\documentstyle[jair,twoside,11pt,theapa,graphicx,psfig]{article}
\input epsf

\def\cutwspace{\setlength{\parskip}{0pt}\setlength{\itemsep}{1pt}}

\newcommand{\progstyle}[1]{ \begin{centering}
\prog{xxxx\=xxxx\=xxxx\=xxxx\=xxxx\=xxxx\= \kill #1}
 \end{centering} }

\newcommand{\prog}[1]{ \begin{centering} \fbox{\ \ \parbox{0pt}{ \begin{tabbing}
#1 \end{tabbing}}\ \ \ } \end{centering}  }

\newcommand{\nnl}{\\[-0.04cm]}

\newtheorem{definition}{Definition}

\newtheorem{theorem}{Theorem}
\newtheorem{lemma}{Lemma}

\jairheading{32}{2008}{169-202}{10/07}{05/08}
\ShortHeadings{Communication-Based Decomposition Mechanism}
{Goldman \& Zilberstein}
\firstpageno{169}

\begin{document}

\title{Communication-Based Decomposition Mechanisms \\ for Decentralized MDPs}

\author{\name Claudia V. Goldman \email c.goldman@samsung.com \\
        \addr Samsung Telecom Research Israel \\
         Yakum, Israel
\AND
\name Shlomo Zilberstein \email shlomo@cs.umass.edu\\
        \addr Department of Computer Science\\ 
         University of Massachusetts, Amherst, MA 01003 USA
}
\maketitle

\begin{abstract}
Multi-agent planning in stochastic environments can be framed formally as a decentralized 
Markov decision problem. Many real-life distributed problems that arise in manufacturing, multi-robot 
coordination and information gathering scenarios can be formalized using this
framework. However, finding the optimal solution in the general case is hard, 
limiting the applicability of recently developed algorithms. 
This paper provides a practical approach for solving decentralized control problems when 
communication among the decision makers is possible, but costly.
We develop the notion of communication-based mechanism that allows us to
decompose a decentralized MDP into multiple single-agent problems. In this framework, referred to 
as decentralized semi-Markov decision process with direct communication (Dec-SMDP-Com), agents operate 
separately between communications. We show that finding an optimal 
mechanism is equivalent to solving optimally a Dec-SMDP-Com. We also provide a heuristic search 
algorithm that converges on the optimal decomposition. 
Restricting the decomposition to some specific types of local 
behaviors reduces significantly the complexity of planning. 
In particular, we present a polynomial-time algorithm for the case in which individual agents perform
goal-oriented behaviors between communications.  The paper concludes 
with an additional tractable algorithm that enables the introduction of
human knowledge, thereby reducing the overall problem to 
finding the best time to communicate.  Empirical results show that these approaches provide 
good approximate solutions.
\end{abstract}

\section{Introduction}

The decentralized Markov decision process has become a common formal tool to 
study multi-agent planning and control from a decision-theoretic perspective~\cite{Bernstein02,Goldman04d,Guestrin02,Guestrin01,Nair03,Petrik07,Kaelbling00}.
Seuken and Zilberstein \citeyear{Seuken08} provide a comprehensive comparison of the
existing formal models and algorithms. Decentralized MDPs complement existing 
approaches to coordination of multiple agents based on on-line learning and 
heuristic approaches~\cite{Wolpert99,Schneider99,Ping01,Mohammed04,Nair04}.

Many challenging real-world problems can be formalized as instances of decentralized MDPs. In 
these problems, exchanging information constantly between the decision makers is 
either undesirable or impossible. Furthermore, 
these processes are controlled by a group of decision makers that must act based on different 
partial views of the global state.  Thus, a centralized approach to action selection is infeasible.
For example, exchanging information with a single central controller can lead to saturation of the 
communication network. Even when the transitions and observations of the agents 
are independent, the global problem may not 
decompose into separate, individual problems, thus a simple parallel algorithm may not be sufficient.
Choosing different local behaviors could lead to different global
rewards. Therefore, agents may need to exchange information periodically and revise their 
local behaviors.  One important point to understand the model we propose is that although
eventually each agent will behave following some local behavior, choosing among possible 
behaviors requires information from other agents.  We focus on situations in which this information 
is not freely available, but it can be obtained via communication.

Solving optimally a general decentralized control problem has been shown to be 
computationally hard~\cite{Bernstein02,Tambe02}.  In the worst case, the general problem requires
a double-exponential algorithm\footnote{Unless NEXP is different from EXP, we cannot prove 
the super-exponential complexity. But, it is generally believed
that NEXP-complete problems require double-exponential time to solve optimally.}. 
This difficulty is due to two main reasons: 1) none of the 
decision-makers has full-observability of the global system and 2) the global performance of 
the system depends on a global reward, which is affected by the agents' behaviors.
In our previous work~\cite{Goldman04c}, we have studied the complexity of solving optimally 
certain classes of Dec-MDPs and Dec-POMDPs\footnote{In Dec-MDPs, the observations of all the agents 
are sufficient to determine the global state, while in Dec-POMDPs the global state cannot be fully determined by the observations.}. 
For example, we have shown that decentralized problems with independent transitions and observations are
considerably easier to solve, namely, they are NP-complete. 
Even in these cases, agents' behaviors can be dependent through the global
reward function, which may not decompose into separate local reward functions.
The latter case has been studied within the context of auction mechanisms for
weakly coupled MDPs by Bererton et al.~\citeyear{Bererton03}.
In this paper, the solution to this type of more complex decentralized 
problems includes temporally abstracted actions combined with communication actions. 
Petrik and Zilberstein~\citeyear{Petrik07} have recently presented
an improved solution to our previous Coverage Set algorithm~\cite{Goldman04d}, 
which can solve decentralized problems optimally. However, the technique is
only suitable when no communication between the agents is possible.
Another recent study by Seuken and Zilberstein~\citeyear{Seuken07a,Seuken07b} 
produced a more general approximation technique based on dynamic programming 
and heuristic search. While the approach shows better scalability, it remains
limited to relatively small problems compared to the decomposition method 
presented here.

We propose an approach to approximate the optimal solutions of decentralized 
problems off-line. The main idea is to compute multiagent macro actions that 
necessarily end with communication. Assuming that communication incurs some
cost, the communication policy is computed optimally, that is the algorithms 
proposed in this paper will compute the best time for the agents to exchange 
information. At these time points, agents attain full knowledge of the current global 
state. These algorithms also compute for each agent what domain actions to 
perform between communication, these are 
temporally abstracted actions that can be interrupted at any time. 
Since these behaviors are computed for each agent separately and independently 
from each other, the final complete solution of communication and action 
policies is not guaranteed to be globally optimal. We refer to this 
approach as a communication-based decomposition mechanism: the algorithms 
proposed compute mechanisms to decompose the global behavior of the agents 
into local behaviors that are coordinated by communication. 
Throughout the paper, these algorithms differ in the space of behaviors in which they search: 
our solutions range from the most general search space available (leading to the optimal mechanism) 
to more restricted sets of behaviors. 

The contribution of this paper is to provide a tractable method, namely 
communication-based decomposition mechanisms, to solve decentralized problems, for which no
efficient algorithms currently exist. For general decentralized problems, our approach
serves as a practical way to approximate the solution in a systematic way. We also provide
an analysis about the bounds of these approximations when the local transitions are not 
independent. For specific cases, like those with independent transitions and observations, we show how to 
compute the optimal decompositions into local behaviors and optimal policies of communication to coordinate
the agents' behaviors at the global level.

Section~\ref{mechanisms} introduces the notion of communication-based mechanisms.
We formally frame this approach as a decentralized semi-Markov decision process with direct 
communication (Dec-SMDP-Com) in Section~\ref{MDPOPT}. 
Section~\ref{PolicyIT} presents the decentralized multi-step backup policy-iteration algorithm that returns
the optimal decomposition mechanism when no restrictions are imposed on the individual behaviors
of the agents. Due to this generality, the algorithm is applicable in some limited domains. 
Section~\ref{DecSMDPLocalGoals} presents a more practical solution, considering that each agent
can be assigned local goal states. 
Assuming local goal-oriented behavior reduces the complexity of the problem to polynomial in the number of 
states. Empirical results (Section~\ref{ProdExp}) support these claims.
Our approximation mechanism can also be applied when the range of possible local behaviors are 
provided at design time. Since these predetermined local behaviors alone may not be sufficient to
achieve coordination, agents still need to decide when to communicate.
Section~\ref{greedy} presents a polynomial-time algorithm that computes the policy 
of communication, given local policies of domain actions. The closer the human-designed local plans are 
to local optimal behaviors, the closer our solution will be to the optimal joint solution. 
Empirical results for the Meeting under Uncertainty scenario (also known as the
Gathering Problem in robotics, \citeauthor{Suzuki99}, 1999) are presented in Section~\ref{example}. 
We conclude with a discussion of the contributions of this work in Section~\ref{discussion}.

\section{The Dec-MDP model}

Previous studies have shown that decentralized MDPs in 
general are very hard to solve optimally and off-line even when direct communication is 
allowed~\cite{Bernstein02,Tambe02,Goldman04c}. 
A comprehensive complexity analysis of solving optimally decentralized control 
problems revealed the sources of difficulty in solving these problems~\cite{Goldman04c}.
Very few algorithms were proposed that can actually solve some classes of problems optimally and
efficiently.

We define a general underlying process which allows agents to exchange messages directly with each other
as a decentralized POMDP with direct communication:
\begin{definition}[Dec-POMDP-Com]
\label{DecPOMDPComDef}
A decentralized partially-observable Markov decision process with direct communication, Dec-POMDP-Com
is given by the following tuple: \newline
$\overline{M}=<S,A_1,A_2,\Sigma,C_\Sigma,P,R,\Omega_1,\Omega_2,O,T>$, where
\begin{itemize}
\setlength{\itemsep}{-2pt}
%\cutwspace
\item $S$ is a finite set of world states, that are factored and include a distinguished initial state 
$s^0$.
\item $A_1$ and $A_2$ are finite sets of actions. $a_{i}$ denotes the 
action performed by agent $i$. 
\item $\Sigma$ denotes the alphabet of messages and $\sigma_i \in \Sigma$ represents an atomic 
message sent by agent $i$ (i.e., $\sigma_i$ is a letter in the language).
\item $C_{\Sigma}$ is the cost of transmitting an atomic message: $C_{\Sigma}: 
\Sigma \rightarrow \Re$. The cost of transmitting a null message is zero.
\item $P$ is the transition probability function.
$P(s'|s,a_1,a_2)$ is the probability of moving from state $s\in S$ to state 
$s'\in S$ when agents $1$ and $2$ perform actions $a_1$ and $a_2$ respectively.
This transition model is stationary, i.e., it is independent of time.
\item $R$ is the global reward function. $R(s,a_1,a_2,s')$
represents the reward obtained by the system as a whole, when agent $1$
executes action $a_1$ and agent $2$ executes action $a_2$ in state $s$ 
resulting in a transition to state $s'$.
\item $\Omega_1$ and $\Omega_2$ are finite sets of observations.
\item $O$ is the observation function. $O(o_1,o_2|s,a_1,a_2,s')$ is the
probability of observing $o_1$ and $o_2$ (respectively by the two agents) when 
in state $s$ agent $1$ takes action $a_1$ and agent $2$ takes action $a_2$, 
resulting is state $s'$. 
\item If the Dec-POMDP has a finite horizon, it is represented by a positive
integer $T$. The notation $\tau$ represents the set of discrete time
points of the process.
\end{itemize}
\end{definition}

The optimal solution of such a decentralized problem is a joint policy
that maximizes some criteria--in our case, the expected accumulated reward
of the system. A joint policy is a tuple composed of local policies for each
agent, each composed of a policy of action and a policy of communication:
i.e., a joint policy $\delta=(\delta_1,\delta_2)$, where
$\delta_i^A:\Omega^*_i \times \Sigma^* \rightarrow A_i$ and 
$\delta_i^\Sigma:\Omega^*_i \times \Sigma^* \rightarrow \Sigma$. 
That is, a local policy of action assigns an action to any possible sequence 
of local observations and messages received. A local policy of communication 
assigns a message to any possible sequence of observations and messages 
received. In each cycle, agents can perform a domain action, then perceive 
an observation and then can send a message.%that is received instantaneously

We assume that the system has independent observations and transitions (see Section~\ref{LGO-nonindependent}
for a discussion on the general case).
Given factored system states $s=(s_1,s_2)\in S$, the domain actions $a_i$ 
and the observations $o_i$ for each agent, the formal definitions\footnote{These definitions are based on Goldman and Zilberstein~\citeyear{Goldman04c}.  We include them here to make the paper self-contained.} for decentralized 
processes with independent transitions, and observations follow. We note that this
class of problems is not trivial since the reward of the system is not necessarily
independent.  For simplicity, we present our definitions for the case of two agents. However, the approach 
presented in the paper is applicable to systems with $n$ agents.

\begin{definition}[A Dec-POMDP with Independent Transitions]
\label{ITDef}
A Dec-POMDP has independent transitions if the set $S$ of states can be 
factored into two components $S = S_1 \times S_2$ such that:
\[\forall s_1,s'_1\!\in\!S_1,\forall s_2,s'_2\!\in\!S_2,\forall a_1\!\in\!A_1,\forall a_2\!\in\!A_2,\]
\[\ \ \ \ \ \ \  Pr(s'_1|(s_1,s_2),a_1,a_2,s'_2) = Pr(s'_1|s_1,a_1) ~\wedge\]
\[\ \ \ \ \ \ \  Pr(s'_2|(s_1,s_2),a_1,a_2,s'_1) = Pr(s'_2|s_2,a_2).\]
In other words, the transition probability $P$ of the Dec-POMDP can be
represented as \newline $P=P_1\cdot P_2$, where 
$P_1=Pr(s'_1|s_1,a_1)$ and $P_2=Pr(s'_2|s_2,a_2)$.
\end{definition}

\begin{definition}[A Dec-POMDP with Independent Observations]
\label{IODef}
A Dec-POMDP has independent observations if the set $S$ of states can be 
factored into two components $S = S_1 \times S_2$ such that:
\[\forall o_1\!\in\!\Omega_1,\forall o_2\!\in\!\Omega_2,\forall s\!=\!(s_1,s_2),s'\!=\!(s'_1,s'_2)\in S,\forall a_1\!\in\!A_1,\forall a_2\!\in\!A_2,\]
\[\ \ \ \ \ \ \  Pr(o_1|(s_1,s_2),a_1,a_2,(s'_1,s'_2),o_2) = Pr(o_1|s_1,a_1,s'_1) \wedge\]
\[\ \ \ \ \ \ \  Pr(o_2|(s_1,s_2),a_1,a_2,(s'_1,s'_2),o_1) = Pr(o_2|s_2,a_2,s'_2) \]
\[O(o_1,o_2|(s_1,s_2),a_1,a_2,(s'_1,s'_2)) = Pr(o_1|(s_1,s_2),a_1,a_2,(s'_1,s'_2),o_2)\cdot Pr(o_2|(s_1,s_2),a_1,a_2,(s'_1,s'_2),o_1).\]
In other words, the observation probability $O$ of the Dec-POMDP 
can be decomposed into two observation probabilities $O_1$ and $O_2$, such that
$O_1=Pr(o_1|(s_1,s_2),a_1,a_2,(s'_1,s'_2),o_2)$ and $O_2=Pr(o_2|(s_1,s_2),a_1,a_2,(s'_1,s'_2),o_1)$.
\end{definition}

\begin{definition}[Dec-MDP]
\label{DecMDPDef}
A decentralized Markov decision process (Dec-MDP) is a Dec-POMDP, which is jointly fully
observable, i.e., the combination of both agents' observations determine the global
state of the system.
\end{definition}
In previous work~\cite{Goldman04c}, we proved that Dec-MDPs with independent transitions and observations are locally fully-observable. In particular, we showed that exchanging the last observation is sufficient to obtain complete 
information about the current global state and it guarantees optimality of the solution.

We focus on the computation of the individual behaviors of the 
agents taking into account that they can exchange information from time to 
time. The following sections present the communication-based decomposition 
approximation method to solve Dec-MDPs with direct communication and independent transitions 
and observations. 

\section{Communication-based Decomposition Mechanism}
\label{mechanisms}

We are interested in creating a mechanism that will tell us what individual behaviors are the most 
beneficial in the sense that these behaviors taken jointly will result in a good approximation
of the optimal decentralized solution of the global system. Notice that even when the system has
a global objective, it is not straightforward to compute the individual behaviors. The decision problem 
that requires the achievement of some global objective does not tell us which local goals each decision 
maker needs to reach in order to maximize the value of a joint policy that reaches the {\em global} objective. 
Therefore, we propose communication-based decomposition mechanisms as a practical 
approach for approximating the optimal joint policy of decentralized control problems.
Our approach will produce two results: 1) a set of temporarily abstracted actions 
for each global state and for each agent and 2) a policy of communication, aimed at synchronizing the
agents'  partial information at the time that is most beneficial to the system.

Formally, a communication-based decomposition mechanism $CDM$ is a function from any 
global state of the decentralized problem to two single agent behaviors or policies: $CDM: S \rightarrow (Opt_1,Opt_2)$.
In general, a mechanism can be applied to systems with $n$ agents, in which case the decomposition of the decentralized process will be into $n$ individual behaviors.
In order to study communication-based mechanisms, we draw an analogy between  
temporary and local policies of actions and options.
%Options~\cite{Sutton99} are closed-loops of actions. 
Options were defined by Sutton et al.~\citeyear{Sutton99} as temporally abstracted actions,
formalized as triplets including a stochastic single-agent policy, a termination condition, and a 
set of states in which they can be initiated:
$opt=<\!\pi:S\times A \rightarrow [0,1],\beta:S^+\rightarrow[0,1],I\subseteq S\!>$.
An option is available in a state $s$ if $s\in I$. 

Our approach considers options with {\em terminal actions} (instead of terminal states). 
Terminal actions were also considered by Hansen and Zhou~\citeyear{Hansen03} in the framework of 
indefinite POMDPs. We denote the domain actions of agent $i$ as $A_i$. The set of terminal 
actions only includes the messages in $\Sigma$.
For one agent, an option is given by the 
following tuple: $opt_i=<\pi:S_i \times \tau\rightarrow A_i\bigcup \Sigma,I\subseteq S_i>$, 
i.e., an option is a  non-stochastic policy from the agent's partial view (local states) and time
to the set of its primitive domain actions and terminal actions. 
The local states $S_i$ are given by the factored representation of the Dec-MDP with independent
transitions and observations. Similarly, the transitions between local states are known since
$P(s'|s,a_1,a_2)=P_1(s'_1|s_1,a_1)\cdot P_2(s'_2|s_2,a_2)$.

In this paper, we concentrate on terminal actions that are necessarily communication actions. 
We assume that all options are terminated whenever at least one of the agents 
initiates communication (i.e., the option of the message sender terminates when it 
communicates and the hearer's option terminates due to this external event).
We also assume that there is joint exchange of messages, i.e., whenever one agent initiates
communication, the global state of the system is revealed to all the agents receiving those
messages: when agent $1$ sends
its observation $o_1$ to agent $2$, it will also receive agent $2$'s observation $o_2$.
This exchange of messages will cost the system only once. Since we focus on finite-horizon processes, 
the options may also be artificially terminated if the time limit of the problem is reached.
The cost of communication $C_\Sigma$ may include, in addition to the actual transmission cost, the cost 
resulting from the time it takes to compute the agents' local policies.

Communication-based decomposition mechanisms enable the agents to operate separately for certain periods 
of time. The question, then, is how to design mechanisms that will approximate 
best the optimal joint policy of the decentralized problem. We distinguish between three cases: 
general options, restricted options, and predefined options.

General options are built from any primitive domain action and communication action 
given by the model of the problem. Searching 
over all possible pairs of local single-agent policies and communication policies built from these
general options will lead to the best approximation. It is obtained when we compute the optimal mechanism 
among all possible mechanisms.  Restricted options limit the space of feasible options to a much smaller set 
defined using certain behavior characteristics.  Consequently, we can obtain mechanisms with lower complexity. Such tractable mechanisms provide
approximation solutions to decentralized problems for which no efficient algorithms currently exist.
Obtaining the optimal mechanism for a certain set of restricted 
options (e.g., goal-oriented options) becomes feasible, as we show in
Sections~\ref{MDPOPT}-\ref{DecSMDPLocalGoals}. 
Furthermore, sometimes, we may consider options that are pre-defined. 
For example, knowledge about effective individual procedures may already exist.
The mechanism approach allows us to combine 
such domain knowledge into the solution of a decentralized problem. In such situations,
where a mapping between global states and single-agent behaviors already exists, the 
computation of a mechanism returns the policy of communication at the meta-level of 
control that synchronizes the agents' partial information. 
In Section~\ref{greedy},  we study a greedy approach 
for computing a policy of communication when knowledge about local behaviors is given.

Practical concerns lead us to the study of communication-based decomposition mechanisms. 
In order to design applicable mechanisms, two desirable properties need to be considered:
\vspace{-8pt}
\begin{itemize}
\setlength{\itemsep}{0pt}
\item {\bf Computational complexity} --- The whole motivation behind the mechanism approach 
is based on the idea that the mechanism itself has low computational complexity. 
Therefore, the computation of the $CDM$ %DCM 
mapping should be practical in the sense that individual behaviors of each agent will have 
complexity that is lower than the 
complexity of the decentralized problem with free communication. 
There is a trade-off between the complexity of computing a mechanism and the global 
reward of the system. There may not be a simple way to split the decentralized process into 
separate local behaviors. The complexity characteristic should be taken into account
when designing a mechanism; different mechanisms can be computed at different 
levels of difficulty.

\item {\bf Dominance} --- A mechanism $CDM_1$ %$DCM_1$ 
dominates another mechanism $CDM_2$ if the global reward attained by $CDM_1$ with some policy
of communication is larger than the global reward attained by $CDM_2$ with 
any communication policy. A mechanism is {\em optimal} for a certain problem 
if there is no mechanism that dominates it.
\end{itemize}

\section{Decentralized Semi-Markov Decision Problems}
\label{MDPOPT}

Solving decentralized MDP problems with a communication-based decomposition mechanism
translates into computing the set of individual and temporally abstracted actions that each
agent will perform together with a policy of  communication that stipulates when to
exchange information. Hereafter, we show how the problem of computing a mechanism can be 
formalized as a semi-Markov decision problem. In particular, the set of basic actions of this
process is composed of the temporally abstracted actions together with the communication actions.
The rest of the paper presents three algorithms aimed at solving this semi-Markov problem optimally.
The algorithms differ in the sets of actions available to the decision-makers, affecting 
significantly the complexity of finding the decentralized solution. 
It should be noted that the optimality of the
mechanism computed is conditioned on the assumptions of each algorithm (i.e., the first algorithm provides the optimal mechanism over all possible options, the second algorithm provides the optimal mechanism when local goals are assumed, and the last algorithm computes the optimal policy of communication assuming that the local behaviors are given).
Formally, a decentralized semi-Markov decision problem with direct communication 
(Dec-SMDP-Com) is given as follows:

\begin{definition}[Dec-SMDP-Com]
\label{Dec-SMDP-Com}
A factored, finite-horizon Dec-SMDP-Com over an underlying Dec-MDP-Com 
$\overline{M}$ is a tuple \newline
$<\overline{M},Opt_1,Opt_2,P^N,R^N>$ where:
\begin{itemize}
\setlength{\itemsep}{0pt}
\item $S$, $\Sigma$, $C_\Sigma$, $\Omega_1$, $\Omega_2$, $P$, $O$ and $T$ are components of the underlying
process $\overline{M}$ defined in definitions~\ref{DecMDPDef} and \ref{DecPOMDPComDef}.
\item $Opt_i$ is the set of actions available to agent $i$. It comprises the possible options that 
agent $i$ can choose to perform, which terminate necessarily with a communication act:\newline
$opt_i=<\pi:S_i \times \tau \rightarrow A_i\bigcup \Sigma,I\subseteq S_i>$.
\item $P^N(s',t\!+\!N|s,t,opt_1,opt_2)$ is the probability of the 
system reaching state $s'$ after exactly $N$ time units, 
when at least one option terminates (necessarily with a communication act).
This probability function is given as part of the model for every value of $N$, such
that $t\!+\!N\leq T$.
In this framework, after $N$ time steps at 
least one agent initiates communication (for the first time since time $t$) and this 
interrupts the option of the hearer agent. 
Then, both agents get full observability of the synchronized state. Since the 
decentralized process has independent transitions and observations, 
$P^N$ is the probability that either agent has communicated or both of them have.
The probability that agent $i$ terminated its option exactly at time $t+N$, $P^{N}_i$, is given as follows:
\begin{small}
\[P^{N}_i(s'_i,t\!+\!N|s_i,t,opt_i) = \left\{ \begin{array}{ll}
                       1  &\hskip-8cm\mbox{ if $(\pi_{opt_i}(s_i,t)\in \Sigma)\wedge (N\!=\!1)\wedge (s'_i\!=\!s_i))$}\\
                       0  &\hskip-8cm\mbox{ if $(\pi_{opt_i}(s_i,t)\in \Sigma)\wedge (N\!=\!1)\wedge (s'_i\!\neq\!s_i))$}\\
                       0  &\hskip-8cm\mbox{ if  $(\pi_{opt_i}(s_i,t)\in A)\wedge (N\!=\!1))$}\\
                       0  &\hskip-8cm\mbox{ if $(\pi_{opt_i}(s_i,t)\in \Sigma)\wedge (N\!>\!1))$}\\
\ & \\
\                         &\hskip-8cm\mbox{ if $(\pi_{opt_i}(s_i,t)\in A)\wedge(N\!>\!1))$}\\
                       \sum_{q_i\in S_i} P_i(q_i|s_i,\pi_{opt_{i}}(s_i,t))P^{N}_i(s'_i,(t\!+\!1)\!+\!(N\!-\!1)|q_i,t\!+\!1,opt_i)\\

                     \end{array}
            \right. \]
\end{small}
The single-agent probability is one when the policy of the option instructs the agent
to communicate (i.e., $\pi_{opt_i}(s_i,t) \in \Sigma$), in which case the local process remains in the same
local state.

We use the notation $s=(s_1,s_2)$ and $s'=(s'_1,s'_2)$ to refer to each agent's local state.
Then, we denote by $\overline{P}^N_i(s'_i,t\!+\!N|s_i,t,opt_i)$ the probability that agent $i$
will reach state $s'_i$ in $N$ time steps when it follows the option $opt_i$. It refers to the probability of reaching some state $s'_i$ without
having terminated the option necessarily when this state is reached. This
transition probability can be computed recursively since the transition probability of 
the underlying Dec-MDP is known: 
\begin{small}
\[\overline{P}^N_i(s'_i,t\!+\!N|s_i,t,opt_i)=\left\{ \begin{array}{ll}
P_i(s'_i|s_i,\pi_{opt_{i}}(s_i,t))) &\hskip-4cm \mbox{if $N\!=\!1$} \\
 \ &\hskip-4cm \mbox{otherwise} \\
\sum_{q_i\in S_i}P(s'_i|q_i,\pi_{opt_{i}}(q_i,t))\overline{P}^N_i(s'_i,(t\!+\!1)\!+\!(N\!-\!1)|s_i,t\!+\!1,opt_i) & \ \\
\end{array}
\right. \]
\end{small}

Finally, we obtain that:
\[P^N(s',t\!+\!N|s,t,opt_1,opt_2)=P^N_1(s'_1,t\!+\!N|s_1,t,opt_1)\cdot \overline{P}^N_2(s'_2,t\!+\!N|s_2,t,opt_2)+\]
\[\ \ \ \ \ \ \ \ P^N_2(s'_2,t\!+\!N|s_2,t,opt_2)\cdot \overline{P}^N_1(s'_1,t\!+\!N|s_1,t,opt_1)\]
\[-P^N_1(s'_1,t\!+\!N|s_1,t,opt_1)\cdot P^N_2(s'_2,t\!+\!N|s_2,t,opt_2)\]

\item $R^N(s,t,opt_1,opt_2,s',t\!+\!N)$ is the expected reward obtained  by the system $N$ 
time steps after the agents started options $opt_1$ and $opt_2$ respectively in state $s$ 
at time $t$, when at least one of them has terminated its option with a communication act
(resulting in the termination of the other agent's option). This reward is computed for $t\!+\!N\leq T$.

\[R^N(s,t,opt_1,opt_2,s',t\!+\!N)= \left\{ \begin{array}{ll}
\overline{C}(opt_1,opt_2,s,s',N)       &\mbox{if $t\!+\!N=T$}\\
\overline{C}(opt_1,opt_2,s,s',N)+C_\Sigma       &otherwise\\
                     \end{array}
            \right. \]

$\overline{C}(opt_1,opt_2,s,s',N)$ is the expected cost incurred by the system when it
transitions between states $s$ and $s'$ and at least one agent communicates after
$N$ time steps.
We define the probability of a certain sequence of global states being transitioned by the
system when each agent follows its corresponding option as $P(<\!s^0,s^1,\ldots,s^N\!>)$:
\[P(<\!s^0,s^1,\ldots,s^N\!>)=\alpha \prod_{j\!=\!0}^{N\!-\!1}P(s^{j\!+\!1}|s^j,\pi_{opt_{1}}(s^j_1),\pi_{opt_{2}}(s^j_2))\]

$\alpha$ is a normalizing factor that makes sure that over all possible sequences, the
probability adds up to one for a given $s^0$, $s^N$ and $N$ steps going through
intermediate steps $s^1,\ldots,s^{N-1}$. 
Then, we denote by $R_{seq}$ the reward attained by the system when it traverses a certain sequence of states.
Formally, $R_{seq}(<\!s^0,\ldots,s^N\!>)=\sum_{j\!=\!0}^{N\!-\!1} R(s^j,\pi_{opt_{1}}(s^j_1),\pi_{opt_{2}}(s^j_2),s^{j\!+\!1})$ where $\pi_{opt_{i}}(s^j_i))$ refers to the primitive action that is chosen by the option at the local state $s^j_i$.
Finally, we can define the expected cost $\overline{C}(opt_1,opt_2,s,s',N)$ as follows:
\[\overline{C}(opt_1,opt_2,s,s',N)=\sum_{q^1,\ldots,q^{N\!-\!1}\in S} P(<\!s,q^1,\ldots,q^{N\!-\!1},s'\!>)R_{seq}(<\!s,q^1,\ldots,q^{N\!-\!1},s'\!>)\]
\end{itemize}
\end{definition}

The dynamics of a semi-Markov decentralized process are as follows. Each agent performs its option 
starting in some global state $s$ that is fully observed. Each agent's option is a mapping from
local states to actions, so agent $i$ starts the option in state $s_i$ at time $t$ until it 
terminates in some state $s'_i$, $k$ time steps later. 
Whenever the options are terminated, the agents can fully observe the global state due to
the terminal communication actions. If they reach 
state $s'$ at time $t\!+\!k\!<\!T$, then the joint policy chooses a possible different pair of 
options at state $s'$ at time $t\!+\!k$ and the process continues. 

Communication in our model leads to a joint exchange of messages. Therefore all the agents observe the 
global state of the system once information is exchanged. This means that all those
states of the decentralized semi-Markov process are {\em fully-observable} (as opposed 
to jointly fully-observable states as in the classical Dec-MDP-Com).

The local policy for agent $i$ in the Dec-SMDP-Com is a mapping from the global
states to its options (as opposed to a mapping from sequences of observations as in the
general Dec-POMDP case, or a mapping from a local state as in the Dec-MDPs with 
independent transitions and observations):
\[\delta_i:S \times \tau \rightarrow Opt_i\]
A joint policy is a tuple of local policies, one for each agent, i.e., a joint policy instructs
each agent to choose an option in each global state. Thus, solving for an optimal mechanism 
is equivalent to solving optimally a decentralized semi-Markov decision problems with temporally 
abstracted actions.

\begin{lemma}
\label{MMDPlemma}
A Dec-SMDP-Com is equivalent to a multi-agent MDP.
\end{lemma}

\begin{proof}
Multiagent MDPs (MMDPs) represent a Markov decision process
that is controlled by several agents~\cite{Boutilier99}. One important feature of this model
is that all the agents have a central view of the global state. Formally, MMDPs 
are tuples of the form $<Ag,\{A_i\}_{i\in Ag},S,Pr,R>$, where:
\begin{itemize}
\cutwspace
\item $Ag$ is a finite collection of $n$ agents.
\item $\{A_i\}_{i\in Ag}$ represents the joint action space. 
\item $S$ is a finite set of system states.
\item $Pr(s'|s,a_1,\ldots,a_n)$ is the transition probability between global 
states $s$ and $s'$ when the agents perform a joint action.
\item $R:S \rightarrow \Re$ is the reward that the system obtains when a global
state is reached.
\end{itemize} 

A decentralized semi-Markov problem with direct communication
can be solved optimally by solving the corresponding MMDP. For simpicity of exposition
we show the proof for systems with two agents. 
Following Definition~\ref{Dec-SMDP-Com}, a 2-agent Dec-SMDP-Com
is given by the tuple: $<\overline{M},Opt_1,Opt_2,P^N,R^N>$.
The mapping between these two models is as follows:
$Ag$ is the same finite collection of agents that control the MMDP and
the semi-Markov process. The set $S$ is the set of states of the world in 
both cases. In the MMDP model, these states are fully observable by 
definition. In the semi-Markov decentralized model these global states are
also fully observable because agents always exchange information at the
end of any option that they perform. 
The set of joint actions $\{A_i\}_{i\in Ag}$ is given in the  
semi-Markov process as the set of options available to each agent
(e.g., if $n=2$ then $\{A_i\}=\{Opt_1,Opt_2\}$). The difference is that the 
joint actions are chosen from primitive domain actions in the MMDP and the options are 
temporarily abstracted actions which terminate with a communication act. 
The probability transition and the reward functions can be easily mapped
between the models by matching $P^N$ with $Pr$ and $R^N$ with $R$.

The solution to an MMDP (or Dec-SMDP-Com) problem is a 
strategy that assigns a joint action (or a set of options) to each global state.
Solving an MMDP with actions given as options
solves the semi-Markov problem. Solving a semi-Markov problem when the options
are of length two, i.e., each option is composed of exactly one primitive action followed by
a communication action that tells the agent to communicate its observation
solves the corresponding MMDP problem.
\hfill $\Box$
\end{proof}

Solving a decentralized semi-Markov process with communication is P-complete because of Lemma~\ref{MMDPlemma}
and the polynomial complexity of single agent MDPs~\cite{Papadimitriou87}. However, the input 
to this problem not only includes the states but also a double exponential number of domain actions for each agent. As explained in the next section, 
each option can be represented as a tree, where: 1) the depth of an option is limited by the finite
horizon $T$ and 2) the branching factor of an option is constrained by the number
of states in $S$. Therefore, the maximal number of leaves an option might have
is bounded by $|S|^T$. Consequently, there can be $|A|^{|S|^T}$ assignments of
primitive domain and communication acts to  the leaves in each possible option.

The naive solution to a Dec-SMDP-Com problem is to search the space of all possible 
pairs of options and find the pair that maximizes the value of each global
state. The multi-step policy-iteration algorithm, presented in Section~\ref{PolicyIT}, 
implements a heuristic version of this search that converges to the optimal mechanism. 
The resulting search space (after pruning) can become intractable for even very simple and 
small problems. Therefore, we propose to apply communication-based decomposition mechanisms
on restricted sets of options. Solving a Dec-SMDP-Com with a restricted set of options means to 
find the optimal policy that attains the maximal value over all possible options in the restricted
set~\cite{Sutton99,Puterman94}. Sections~\ref{DecSMDPLocalGoals} and \ref{greedy} present
two additional algorithms that solve Dec-SMDP-Com problems when the options considered
are goal-oriented options, i.e., the mechanism assigns local goals to each one of the agents at each
global state, allowing them to communicate before having reached their local goals.

\section{Multi-step Backup Policy-Iteration for Dec-SMDP-Com}
\label{PolicyIT}

Solving a Dec-SMDP problem optimally means computing the optimal pair of options for each 
fully-observable global state. These options instruct the agents how to act independently 
of each other until information is exchanged. In order to find these options for 
each global state, we apply an adapted and extended version of the multi-step backup 
policy-iteration algorithm with heuristic search~\cite{HansenTechRep97}. 
We show that the decentralized version of this algorithm converges to the optimal policy of 
the decentralized case with temporally abstracted actions. 

We extend the model of the single-agent POMDP with observations costs to the
Dec-SMDP-Com model. From a global perspective, each agent that follows its own option 
without knowing the global state of the system, is following an open-loop policy.
However, locally, each agent is following an option, which does depend on the agent's 
local observations.
We first define a {\em multi-step backup} for options, when $s$ and $s'$ are global 
states of the decentralized problem: 
\begin{small}
$V(s,t,T)=$
\[\ \ \max_{opt_1,opt_2\in \mbox{OPT}_b}\{\sum_{k\!=\!1}^{min\{b,T\!-\!t\}}\sum_{s'}P^N(s',t\!+\!k|s,t,opt_1,opt_2)
[R^N(s,t,opt_1,opt_2,s',t\!+\!k)+V(s',t\!+\!k,T)]\}\]
\end{small}
$\mbox{OPT}_b$ is the set of options of length at most $b$, where the length is defined as follows:

\begin{definition} [The length of an Option]
The length of an option is $k$ if the option can perform at most $k$ domain actions
in one execution.
\end{definition}

As in Hansen's work, $b$ is a bound on the length of the options ($k\leq b$). 
Here, the finite horizon Dec-SMDP-Com case is analyzed, therefore $b\leq T$.
$P^N(s',t\!+\!k|s,t,opt_1,opt_2)$ and $R^N(s,t,opt_1,opt_2,s',t\!+\!k)$ are taken from 
the Dec-SMDP-Com model (Definition~\ref{Dec-SMDP-Com}). 

We apply the multi-step backup policy-iteration algorithm (see Figure~\ref{MSBPIalgorithm}) using 
the pruning rule introduced by Hansen~\citeyear{HansenTechRep97}, which we
adapt to work on pairs of policies instead of 
linear sequences of actions. The resulting optimal multi-step backup policy is equivalent 
to the optimal policy of the MMDP (Lemma~\ref{MMDPlemma}), i.e., it is equivalent to the optimal 
decentralized policy of a Dec-SMDP-Com with temporally abstracted actions.
In order to explain the pruning rule for the decentralized case with temporally abstracted actions, we 
define what policy-tree structures are. 

\begin{definition}[Policy-tree]
A {\em policy-tree} is a tree structure, composed of local state nodes and corresponding action nodes
at each level. Communication actions can only be assigned to leaves of the tree.
The edges connecting an action $a$ (taken at the parent state $s_i$) with a resulting state 
$s'_i$ have the transition probability $P_i(s'_i|s_i,a)$ assigned to them.
\end{definition}

Figure~\ref{policytree} shows a possible policy-tree. An option is represented by a 
policy-tree with all its leaves assigned communication actions.
\begin{figure}[t]
%\centerline{\psfig{figure=PolicyTree.pdf,height=7cm}}
\centerline{\includegraphics[height=2in]{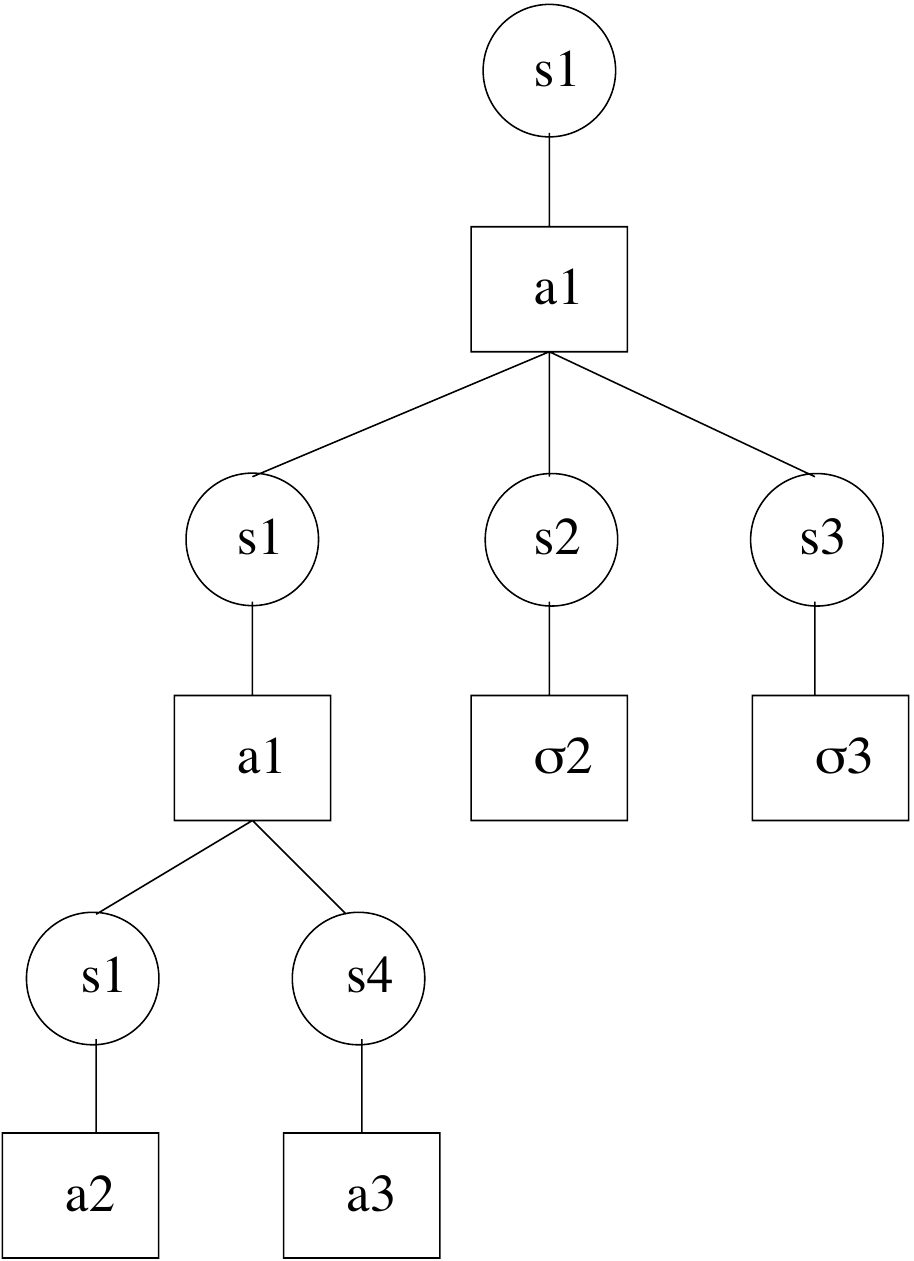}}
\caption{A policy tree of size $k$=3.}
\label{policytree}
\end{figure}
We denote a policy-tree $s_{root}\alpha$, by the state assigned to its root 
(e.g., $s_{root}$), and an  assignment of domain actions and local states to the rest of 
the nodes (e.g., $\alpha$).
The size of a policy-tree is defined as follows:
\begin{definition}[Size of a Policy-tree]
The size of a policy-tree is $k$ if the longest branch of the tree, starting from 
its root is composed of $k\!-\!1$ edges (counting the edges between actions and resulting
states).
A policy-tree of size one includes the root state and the action taken at that state.
\end{definition}

$\pi(\alpha_k)$ is the policy induced by the assignment $\alpha$ with at most $k$ actions in 
its implementation. The {\em expected cost $g$ of a policy-tree $s_{root}\alpha_k$} is 
the expected cost that will be incurred by an agent when it follows the policy 
$\pi(\alpha_k)$. We denote the set of nodes in a tree that do not correspond to leaves as $NL$ and
the set of states assigned to them $S_{NL}$. 
%If the size of a policy-tree is $k$, then there can be at least $k$ elements in $S_{NL}$. 
The notation $\alpha\setminus n$ refers to the $\alpha$ assignment excluding node $n$. 
The expected cost of a tree, $g(s_{root}\alpha_{k})$, is computed as follows: 

\[g(s_{root}\alpha_k)= \left\{ \begin{array}{ll}
                        C(a_{root}) &\mbox{if $k=1$}\\
                        C(a_{root}) + \sum_{s'_i\in S_{NL}} [Pr(s'_i|s_{root},a_{root})g(s'_i(\alpha\setminus root)_{k-1})] &\mbox{if $1< k\leq T$}\\
                       \end{array}
                       \right. \]

Since the decentralized process has factored states, we can write a global state
$s$ as a pair $(s_1,s_2)$. Each agent can act independently of each other for 
some period of time $k$ while it performs an option. Therefore, we can refer to the information state of the system after 
$k$ time steps as $s_{1}\alpha_k\!s_{2}\beta_k$, where $s_1\alpha_k$ and $s_2\beta_k$ 
correspond to each agent's policy tree of size $k$. We assume that at least one
agent communicates at time $t\!+\!k$. This will necessarily interrupt the other agent's
option at the same time $t\!+\!k$. Therefore, it is sufficient to look at pairs
of trees of the same size $k$. The information state refers to the 
belief an agent forms about the world based on the partial information 
available to it while it operates locally. In our model, agents get full
observability once they communicate and exchange their observations.

The heuristic function that will be used in the search for the optimal decentralized joint 
policy of the Dec-SMDP-Com follows the traditional notation,
i.e., $f(s) = g(s)+h(s)$. In our case, these functions will be defined over pairs of
policy-trees, i.e., $f(s\alpha_k\beta_k)=G(s\alpha_k\beta_k)+H(s\alpha_k\beta_k)$. The $f$ value 
denotes the backed-up value for implementing policies $\pi(\alpha_k)$ and $\pi(\beta_k)$, 
respectively by the two agents, starting in state $s$ at time $t$. The expected value of a 
state $s$ at time $t$ when the horizon is $T$ is given by the multi-step backup for state 
$s$ as follows: \[V(s,t,T)=\max_{|\alpha|,|\beta| \leq \ b }\{f(s\alpha\beta)\}.\] 
Note that the policy-trees corresponding to the assignments $\alpha$ and $\beta$ are of size 
{\em at most} $b\leq T$.  We define the expected cost of implementing a pair of policy-trees,
$G$, as the sum of the expected costs of each one separately.
If the leaves have communication actions, the cost of communication is taken into account
in the $g$ functions. 
As in Hansen's work, when the leaves are not assigned a communication action, we 
assume that the agents can sense at no cost to compute the $f$ function.
\[G(s_{1}\alpha_k\!s_{2}\beta_k)=g(s_{1}\alpha_k)+g(s_{2}\beta_k).\]

An option is a policy-tree with communication actions assigned to all its leaves.
That option is denoted by $opt_1(\alpha_k)$ (or $opt_2(\beta_k)$).
The message associated with a leaf corresponds to the local state that is assigned to that leaf by $\alpha$ 
(or $\beta$). We define the expected value of perfect information of the information state $s\alpha\beta$ 
after $k$ time steps:
\[H(s\alpha_k\beta_k)=\sum_{s'\in S} P^N(s',t\!+\!k|s,t,opt_1(\alpha_k),opt_2(\beta_k))V(s',t\!+\!k,T)\]

The multi-step backup policy-iteration algorithm adapted from Hansen to the 
decentralized control case appears in Figure~\ref{MSBPIalgorithm}. 
Intuitively, the heuristic search over all possible options unfolds as follows:
Each node in the search space is composed of two policy-trees, each representing a local policy 
for one agent. The search advances through nodes whose $f$ value (considering both
trees) is greater than the value of the global root state (composed of the roots of 
both policy-trees). All nodes whose $f$ value does not follow this inequality are actually pruned 
and are not used for updating the joint policy. The policy is updated when a node, composed
of two options is found for which $f>V$. All the leaves in these options (at all possible depths)
include communication acts. The updated policy $\delta'$ maps the global state $s$ to these two
options. When all the leaves in one policy-tree at current depth $i$ 
have communication actions assigned, the algorithm assigns communication acts to all the leaves in 
the other policy-tree at this same depth. This change in the policies is correct because there
is joint exchange  of information (i.e., all the actions are interrupted when at least one agent communicates). 
We notice, though, that there may be leaves in these policy-trees at depths lower than $i$ 
that may still have domain actions assigned. Therefore, these policy-trees cannot be considered options yet 
and they remain in the stack. Any leaves that remain assigned to domain actions will be expanded by
the algorithm. This expansion requires the addition of all the possible next states, that are reachable
by performing the domain-actions in the leaves, and the addition of a possible action for each 
such state. If all the leaves at depth $i$ of one policy-tree are already assigned communication acts, then
the algorithm expands only the leaves with domain actions at lower levels in both policy-trees. 
No leaf will be expanded beyond level $i$ because at the corresponding time one agent is going to
initiate communication and this option is going to be interrupted anyways.

\begin{figure}
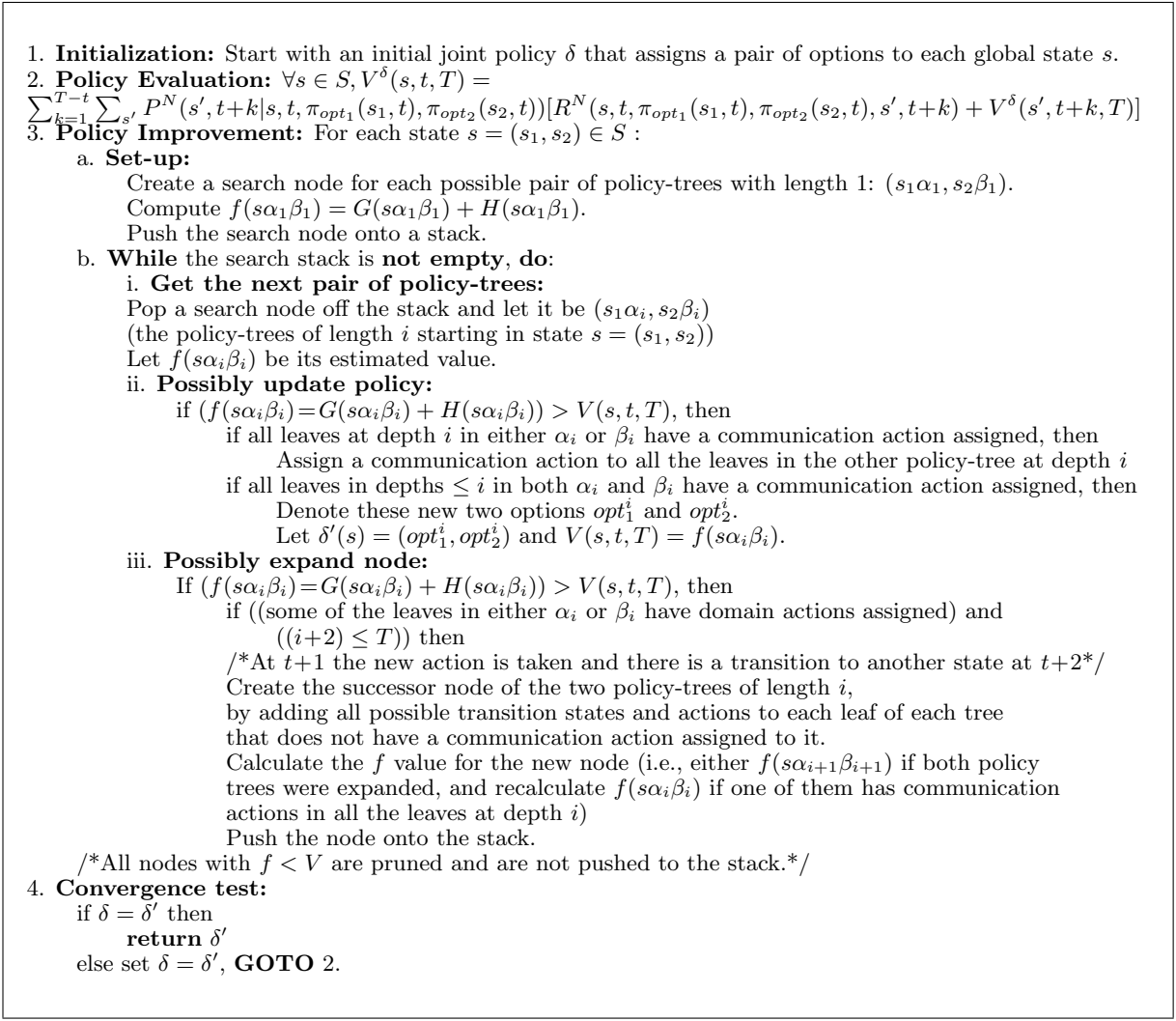

\begin{footnotesize}
\centerline{
\hbox{
\progstyle{
1. {\bf Initialization:} Start with an initial joint policy $\delta$ that assigns a pair 
of options to each global state $s$. \nnl
2. {\bf Policy Evaluation:} $\forall s\in S, V^\delta(s,t,T)=$ \nnl
$\sum_{k=1}^{T-t}\sum_{s'} P^N(s',t\!+\!k|s,t,\pi_{opt_1}(s_1,t),\pi_{opt_2}(s_2,t))[R^N(s,t,\pi_{opt_1}(s_1,t),\pi_{opt_2}(s_2,t),s',t\!+\!k)+V^\delta(s',t\!+\!k,T)]$  \nnl
3. {\bf Policy Improvement:} For each state $s=(s_1,s_2)\in S$ :\nnl   
\> a. {\bf Set-up:} \nnl
\>\>Create a search node for each possible pair of policy-trees with length $1$: $(s_{1}\alpha_1,s_{2}\beta_1)$.  \nnl
\>\> Compute $f(s\alpha_{1}\beta_{1})=G(s\alpha_{1}\beta_{1})+H(s\alpha_{1}\beta_{1})$. \nnl
\>\> Push the search node onto a stack. \nnl
\> b. {\bf While} the search stack is {\bf not empty}, {\bf do}: \nnl
\>\> i. {\bf Get the next pair of policy-trees:} \nnl
\>\> Pop a search node off the stack and let it be $(s_{1}\alpha_{i},s_{2}\beta_i)$ \nnl
\>\> (the policy-trees of length $i$ starting in state $s=(s_1,s_2)$)\nnl
\>\> Let $f(s\alpha_{i}\beta_{i})$ be its estimated value. \nnl
\>\>  ii. {\bf Possibly update policy:} \nnl
\>\>\> if $(f(s\alpha_{i}\beta_{i})\!=\!G(s\alpha_{i}\beta_{i})+H(s\alpha_{i}\beta_{i})) > V(s,t,T)$, then \nnl
\>\>\>\> if all leaves at depth $i$ in either $\alpha_{i}$ or $\beta_{i}$ have a communication action assigned, then \nnl
\>\>\>\>\> Assign a communication action to all the leaves in the other policy-tree at depth $i$\nnl
\>\>\>\> if all leaves in depths $\leq i$ in both $\alpha_i$ and $\beta_i$ have a communication action assigned, then \nnl
\>\>\>\>\> Denote these new two options $opt_1^{i}$ and $opt_2^{i}$. \nnl
\>\>\>\>\> Let $\delta'(s)=(opt_1^{i},opt_2^{i})$ and $V(s,t,T)=f(s\alpha_{i}\beta_{i})$.\nnl
\>\>  iii. {\bf Possibly expand node:} \nnl
\>\>\> If $(f(s\alpha_{i}\beta_{i})\!=\!G(s\alpha_{i}\beta_{i})+H(s\alpha_{i}\beta_{i})) > V(s,t,T)$, then \nnl
\>\>\>\> if ((some of the leaves in either $\alpha_i$ or $\beta_i$ have domain actions assigned) and \nnl
\>\>\>\>\> $((i\!+\!2)\leq T)$) then \nnl
\>\>\>\>/*At $t\!+\!1$ the new action is taken and there is a transition to another state at $t\!+\!2$*/\nnl 
\>\>\>\> Create the successor node of the two policy-trees of length $i$, \nnl
\>\>\>\> by adding all possible transition states and actions to each leaf of each tree \nnl
\>\>\>\> that does not have a communication action assigned to it. \nnl
\>\>\>\> Calculate the $f$ value for the new node (i.e., either $f(s\alpha_{i+1}\beta_{i+1})$ if both policy \nnl
\>\>\>\> trees were expanded, and recalculate $f(s\alpha_{i}\beta_{i})$ if one of them has communication \nnl
\>\>\>\> actions in all the leaves at depth $i$) \nnl
\>\>\>\> Push the node onto the stack. \nnl
\>/*All nodes with $f<V$ are pruned and are not pushed to the stack.*/\nnl
4. {\bf Convergence test:} \nnl
\> if $\delta=\delta'$ then \nnl
\>\> {\bf return} $\delta'$ \nnl
\> else set $\delta=\delta'$, {\bf GOTO} 2.
}%progstyle
}
}
\caption{Multi-step Backup Policy-iteration ($\mbox{MSBPI}$) using depth-first branch-and-bound.}
\label{MSBPIalgorithm}
\end{footnotesize}
\end{figure}

% Hansen~\citeyear{HansenTechRep97} proved that multi-step backup policy-iteration with heuristic pruning converges to the optimal policy in the single-agent case with linear sequences of actions and infinite horizon. 
In the next section, we show the convergence of our Multi-step Backup Policy-iteration ($\mbox{MSBPI}$) algorithm  to the optimal decentralized solution of the Dec-SMDP-Com, when agents 
follow temporally abstracted actions and the horizon is finite.

\subsection{Optimal Decentralized Solution with Multi-step Backups}

In this section, we prove that the $\mbox{MSBPI}$ algorithm presented in Figure~\ref{MSBPIalgorithm} 
converges to the optimal decentralized control joint policy with temporally abstracted actions and
direct communication. We first show that the policy improvement step in the algorithm based on
heuristic multi-step backups improves the value of the current policy if it is sub-optimal.
Finally, the policy iteration algorithm iterates over improving policies and it is known to 
converge.

\begin{theorem}
When the current joint policy is not optimal, the policy improvement step in the multi-step 
backup policy-iteration algorithm always finds an improved joint policy.
\end{theorem}

\begin{proof}
We adapt Hansen's proof to our decentralized control problem, when policies are represented by
policy-trees. Algorithm $\mbox{MSBPI}$ in Figure~\ref{MSBPIalgorithm} updates the current policy 
when the new policy assigns a pair of options that yield a greater value for a certain global state. 
We show by induction on the size of the options, that at least for one state, a new 
option is found in the improvement step (step {\em 3.b.ii}).

If the value of any state can be improved by two policy-trees of size one, then
an improved joint policy is found because all the policy-trees of size one are evaluated.
We initialized $\delta$ with such policy-trees. 
We assume that an improved joint policy can be found with policy-trees of size at most $k$.
We show that an improved joint policy is found with policy-trees of size $k$.
Lets assume that $\alpha_k$ is a policy tree of size $k$, such that $f(s\alpha_k\beta)>V(s)$
with communication actions assigned to its leaves. If this is the case then the policy 
followed by agent $2$ will be interrupted at time $k$ at the latest.
One possibility is that $s\alpha_k\beta$ is evaluated by the algorithm. Then, an
improved joint policy is indeed found.
If this pair of policy-trees was not evaluated by the algorithm, it means that $\alpha$
was pruned earlier. We assume that this happened at level $i$. This means that
$f(s\alpha_i\beta)<V(s)$. We assumed that $f(s\alpha_k\beta)>V(s)$ so we obtain that:
$f(s\alpha_k\beta) > f(s\alpha_i\beta)$.

If we expand the $f$ values in this inequality, we obtain the following:
\begin{small}
\[ g(s\alpha_i)+g(s\beta)+\sum_{s'}P^N(s',t\!+\!i|s,opt_1(\alpha_i),opt_2(\beta))[g(s'\alpha(i,k))+g(s'\beta)+\sum_{s''}P^N(s'',t\!+\!i\!+\!k\!-\!i)V(s'')]>\]
\[g(s\alpha_i)+g(s\beta)+\sum_{s'}P^N(s',t\!+\!i|s,opt_1(\alpha_k),opt_2(\beta))V(s',t\!+\!i)\]
\end{small}
where $g(s'\alpha(i,k))$ refers to the expected cost of the subtree starting at level $i$ and ending at level $k$ starting from $s'$.  After simplification we obtain:
\begin{small}
\[\sum_{s'}P^N(s',t\!+\!i|s,opt_1(\alpha_i),opt_2(\beta))[g(s'\alpha(i,k))+g(s'\beta)+\sum_{s''}P^N(s'',t\!+\!i\!+\!k\!-\!i)V(s'')]>\]
\[\sum_{s'}P^N(s',t\!+\!i|s,opt_1(\alpha_k),opt_2(\beta))V(s',t\!+\!i) \]
\end{small}

That is, there exists some state $s'$ for which $f(s'\alpha(i,k)\beta) > V(s')$.
Since the policy-tree $\alpha(i,k)$ has size less than $k$, by the induction assumption
we obtain that there exists some state $s'$ for which the multi-step backed-up value
is increased. Therefore, the policy found in step $3.b.ii$ is indeed an improved policy.
\hfill $\Box$
\end{proof}

\begin{lemma}
The complexity of computing the optimal mechanism over general options by the $\mbox{MSBPI}$ algorithm is
$O(((|A_1|+|\Sigma|)(|A_2|+|\Sigma|))^{|S|^{(T-1)}})$. (General options are based on any possible primitive
domain action in the model, and any communication act).
\end{lemma}

\begin{proof}
Each agent can perform any of the primitive domain actions in $A_i$ and can communicate
any possible message in $\Sigma$. There can be at most $|S|^{(T-1)}$ leaves in a policy tree with 
horizon $T$ and $|S|$ possible resulting states from each transition. 
Therefore, each time the $\mbox{MSBPI}$ algorithm expands a policy tree (step
3.b.iii in Figure~\ref{MSBPIalgorithm}), the number of resulting trees is 
$((|A_1|+|\Sigma|)(|A_2|+|\Sigma|))^{|S|^{(T-1)}}$. In the worst case, this is the number of trees
that the algorithm will develop in one iteration. Therefore, the size of the search space is 
a function of this number times the number of iterations until convergence.
\hfill $\Box$
\end{proof}

Solving optimally a Dec-MDP-Com with independent transitions
and observations has been shown to be in NP~\cite{Goldman04c}. 
As we show here, solving for the optimal mechanism is 
harder, although the solution may not be the optimal.
This is due to the main difference between these two problems. In the Dec-MDP-Com, we 
know that since the transitions and observations are independent, a local state is a sufficient 
statistic for the history of observations. However, in order to compute an optimal mechanism
we need to search in the space of options, that is, no single local state is a sufficient statistic.
When options are allowed to be general, the search space is larger since each possible option that needs to be considered can be
arbitrarily large (with the length of each branch bounded by $T$).
For example, in the Meeting under Uncertainty scenario (presented in Section~\ref{example}),
agents aim at meeting in a stochastic environment in the shortest time as possible.
Each agent can choose to perform anyone of six primitive actions (four move actions, 
one stay action and a communication action). Even in a small world composed of 
$100$ possible locations,
implementing the $\mbox{MSBPI}$ algorithm is intractable. It will require the 
expansion of all the possible combinations of pairs of policy-trees leading to a 
possible addition of $36^{100^{(T-1)}}$ nodes to the search space at each iteration.
Restricting the mechanism to a certain set of possible options, for example goal-oriented options
leads to a significant reduction in the complexity of the algorithm as we shown in the 
following two sections.

\section{Dec-SMDP-Com with Local Goal-oriented Behavior}
\label{DecSMDPLocalGoals}

The previous section provided an algorithm that computes the optimal mechanism, searching
over all possible combinations of domain and communication actions for each agent.
On the one hand, this solution is the most general and does not restrict the individual behaviors
in any aspect. On the other hand, this solution may require the search of a very large
space, even after this space is pruned by the heuristic search technique.
Therefore, in order to provide a practical decomposition mechanism algorithm, it
is reasonable to restrict the mechanism to certain sets of
individual behaviors. In this section, we concentrate on {\em goal-oriented options}
and propose an algorithm that computes the optimal mechanism with respect to this
set of options: i.e., the algorithm finds a mapping from each global
state to a set of locally goal oriented behaviors with the highest value. The algorithm proposed has the 
same structure as the $\mbox{MSBPI}$ algorithm; the main difference is in how the options are built.

\begin{definition}[Goal-oriented Options]
A goal-oriented option is a local policy that achieves a given local goal.
\end{definition}

We study locally goal-oriented mechanisms, which map each global state to a pair of goal-oriented options 
and a period of time $k$. We assume here that a set of local goals $\hat{G_i}$
is provided with the problem. For each such local goal, a local policy that can achieve it is considered a goal-oriented option. 
When the mechanism is applied, each agent follows its policy to the corresponding local goal for 
$k$ time steps. 
At time $k+1$, the agents exchange information and stop acting (even though they may not 
have reached their local goal). The agents, then, become synchronized and they are 
assigned possibly different local goals and a working period $k'$.

The algorithm presented in this section, $\mbox{LGO-MSBPI}$, solves the decentralized 
control problem with  communication by finding the optimal mapping between global states to 
local goals and periods of time (the algorithm finds the best solution it can 
given that the agents will act individually for some periods of time). 
We start with an arbitrary joint policy that assigns one pair of local goal states and a 
number $k$ to each global state. The current joint policy is evaluated and set as the current best known 
mechanism. Given a joint policy $\delta:S\times \tau \rightarrow \hat{G}_1\times \hat{G}_2\times \tau$, 
($\hat{G}_i \subseteq S_i \subset S$),
the value of a state $s$ at a time $t$, when $T$ is the finite horizon is given in Equation~\ref{VPolicyLG}: 
(this value is only computed for states in which $t\!+\!k\leq T$).
\begin{small}
\begin{eqnarray}
\label{VPolicyLG}
\!\!\!\!\!\!V^\delta(s,t,T) & \!\!\!=\!\!\! & \left\{ \begin{array}{ll}
0 &\hskip-7cm\mbox{if $t\!=\!T$}\\
\sum_{s'\in S} P^N_g(s',t\!+\!k|s,t,\pi_{\hat{g_1}}(s_1),\pi_{\hat{g_2}}(s_2))[R^N_g(s,t,\pi_{\hat{g_1}},\pi_{\hat{g_2}},s',k)\!+\!V^\delta(s',t\!+\!k,T)] \\
&\hskip-4cm\mbox{s.t. $\delta(s,t)=(\hat{g_1},\hat{g_2},k)$}\\
                             \end{array}
                      \right.
\end{eqnarray}
\end{small}

Notice that $R^N_g(s,\pi_{\hat{g_1}},_{\hat{g_2}},s',k)$ can be defined similarly to 
$R^N()$ (see Definition~\ref{Dec-SMDP-Com}),
taking into account that the options here are aimed at reaching a certain 
local goal state ($\pi_{\hat{g_1}}$ and $\pi_{\hat{g_2}}$ are
aimed at reaching the local goal states $\hat{g_1}$ and $\hat{g_2}$, 
respectively).
\begin{small}
\[R^N_g(s,t,\pi_{\hat{g_1}},\pi_{\hat{g_2}},s',k)=\overline{C}(\pi_{\hat{g_1}},\pi_{\hat{g_2}},s,s',k)\!+\!C_{\Sigma}=\]
\[C_{\Sigma}\!+\!\sum_{q^1,\ldots,q^{k\!-\!1}}P(<\!s,q^1,\ldots,q^{k\!-\!1},s'\!>)\cdot C_{seq}(<\!s,q^1,\ldots,s^{k\!-\!1},s'\!>\]
\end{small}

There is a one-to-one mapping between goals and goal-oriented options. That is, the policy $\pi_{{g_i}}$  assigned by $\delta$
can be found by each agent independently by solving optimally each agent's local process 
$MDP_i=(S_i,P_i,R_i,\hat{G}_i,T)$: The set of global states $S$ is factored so
each agent has its own set of local states. The process
has independent transitions, so $P_i$ is the primitive transition probability 
known when we described the options framework. 
$R_i$ is the cost incurred by an agent when it performs a primitive action $a_i$ and zero if the agent
reaches a goal state in $\hat{G}_i$. $T$ is the finite horizon of the global problem. 

$P^N_g$ (with the goal $g$ subscript) is different from the probability function $P^N$ 
that appears in Section~\ref{MDPOPT}. $P^N_g$ is the probability of reaching a global state 
$s'$ after $k$ time steps, while trying to reach $\hat{g}_1$ and $\hat{g}_2$ respectively following
the corresponding optimal local policies. 

\begin{small}
\[P^{N}_g(s',t\!+\!k|s,t\!+\!i,\pi_{\hat{g_1}}(s_1),\pi_{\hat{g_2}}(s_2)) = \]
\[               \left\{ \begin{array}{ll}
                       1  &\hskip-6cm\mbox{if $i\!=\!k$ and $s\!=\!s'$}\\
                       0  &\hskip-6cm\mbox{if $i\!=\!k$ and $s\neq s'$}\\
                       \  &\hskip-6cm\mbox{if $i<k$}\\
\ & \ \\
%                       \sum_{s^*\in S} P(s^*|s,\pi_{\hat{g_1}_\delta}(s_1),\pi_{\hat{g_2}_\delta}(s_2))\cdot P^N_g(s',t\!+\!k_\delta|s^*,t\!+\!i\!+\!1,\pi_{\hat{g_1}_\delta}(s^*_1),\pi_{\hat{g_2}_\delta}(s^*_2)) \\
                       \sum_{s^*\in S} P(s^*|s,\pi_{\hat{g_1}}(s_1),\pi_{\hat{g_2}}(s_2))\cdot P^N_g(s',t\!+\!k|s^*,t\!+\!i\!+\!1,\pi_{\hat{g_1}}(s^*_1),\pi_{\hat{g_2}}(s^*_2)) \\
&\hskip-4cm\mbox{s.t. $\delta(s,t\!+\!i)=(\hat{g_1},\hat{g_2},k)$}\\
                              \end{array}
                      \right.
\]
\end{small}

Each iteration of the $\mbox{LGO-MSBPI}$ algorithm (shown in Figure~\ref{algorithmGoals1}) tries to improve 
the value of each state by testing all the possible pairs of local goal states with 
increasing number of time steps allowed until communication. The value of $f$ is computed for 
each mapping from states to assignments of local goals and periods of time.
The $f$ function for a given global state, current time, pair of local goals and a given
period of time $k$ expresses the cost incurred by the agents after having acted for $k$
time steps and having communicated at time $k+1$, and the expected value of the reachable
states after $k$ time steps (these states are those reached by the agents while
following their corresponding optimal local policies towards $\hat{g}_1$ and $\hat{g}_2$
respectively). The current joint policy is updated when the $f$ value 
for some state $s$, time $t$, local goals $\hat{g}_1$ and $\hat{g}_2$ and period $k$ is greater than 
the value $V^\delta(s,t,T)$ computed for the current best known assignment of local goals 
and period of time. % $k$.
Formally:
\begin{eqnarray}
\label{fLG}
f(s,t,\hat{g}_1,\hat{g}_2,k) &=& G(s,t,\hat{g}_1,\hat{g}_2,k) + H(s,t,\hat{g}_1,\hat{g}_2,k) \\
\label{GLG} G(s,t,\hat{g}_1,\hat{g}_2,k) &=& \overline{C}(\pi_{\hat{g_1}},\pi_{\hat{g_2}},s,t,k) + C_\Sigma \\
%2k C(a)+C_\Sigma \\
\label{HLG} H(s,t,\hat{g}_1,\hat{g}_2,k) &=& \left\{ \begin{array}{ll}
0 &\hskip-2cm\mbox{if $t\!=\!T$}\\
\    &\hskip-2cm\mbox{if $t\!<\!T$} \\ 
\sum_{s'\in S} P^N_g(s',t+k|s,t,\pi_{\hat{g}_1}(s_1),\pi_{\hat{g}_2}(s_2))V^\delta(s',t+k,T) &\mbox{}\\

                              \end{array}
                      \right.
\end{eqnarray}

$\overline{C}(\pi_{\hat{g_1}},\pi_{\hat{g_2}},s,t,k)$ is the expected cost incurred by the 
agents when following the corresponding options for $k$ time steps starting from state
$s$. This is defined similarly to the expected cost explained in Definition~\ref{Dec-SMDP-Com}.
We notice that the computation of $f$ refers to the goals being evaluated by the 
algorithm, while the evaluation of the policy (step $2$) refers to the goals {\em assigned} by the current best policy.

\begin{figure}
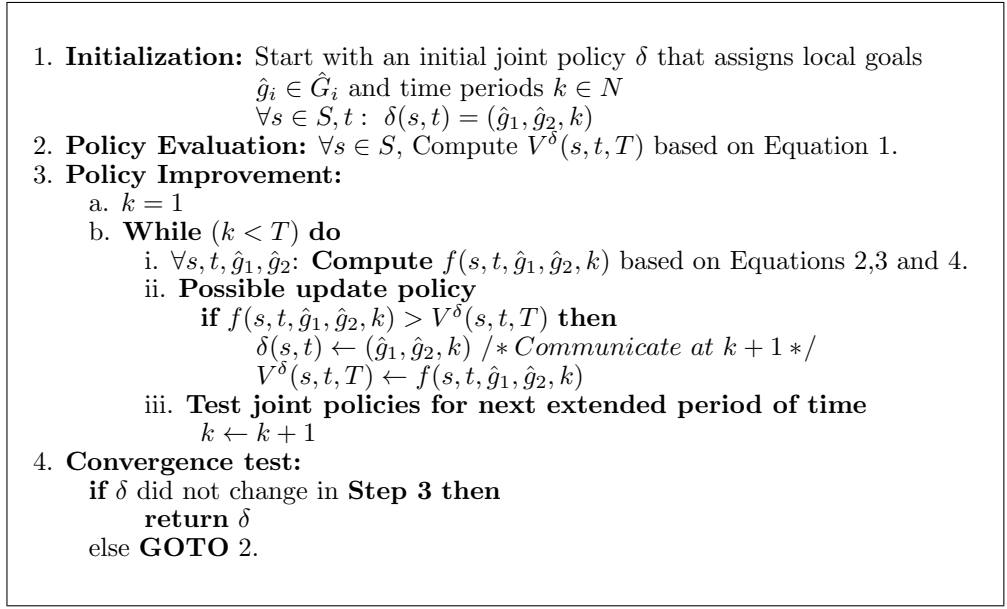

\begin{small}
\centerline{
\hbox{
\progstyle{
1. {\bf Initialization:} Start with an initial joint policy $\delta$ that assigns local goals \nnl
\>\>\>\> $\hat{g}_i\in \hat{G}_i$ and time periods $k\in N$\nnl
\>\>\>\> $\forall s\in S,t: \ \delta(s,t)=(\hat{g}_1,\hat{g}_2,k)$ \nnl
2. {\bf Policy Evaluation:} $\forall s\in S$, Compute $V^\delta(s,t,T)$ based on 
Equation~\ref{VPolicyLG}.\nnl
3. {\bf Policy Improvement:} \nnl  
\> a. $k=1$ \nnl
\> b. {\bf While} $(k < T)$ {\bf do}  \nnl
\>\> i. $\forall s,t,\hat{g}_1,\hat{g}_2$: {\bf Compute} $f(s,t,\hat{g}_1,\hat{g}_2,k)$ based on Equations~\ref{fLG},\ref{GLG} and \ref{HLG}.\nnl
\>\> ii. {\bf Possible update policy} \nnl
\>\>\>  {\bf if} $f(s,t,\hat{g}_1,\hat{g}_2,k) > V^\delta(s,t,T) $ {\bf then} \nnl
\>\>\>\>  $\delta(s,t) \leftarrow (\hat{g}_1,\hat{g}_2,k)$ $/\!* Communicate\ at\ k+1 *\!/$ \nnl
\>\>\>\> $V^\delta(s,t,T) \leftarrow f(s,t,\hat{g}_1,\hat{g}_2,k)$  \nnl
\>\>  iii. {\bf Test joint policies for next extended period of time}\nnl
\>\>\> $k \leftarrow k+1$ \nnl
4. {\bf Convergence test:} \nnl
\> {\bf if} $\delta$ did not change in {\bf Step 3} {\bf then} \nnl
\>\> {\bf return} $\delta$ \nnl
\> else {\bf GOTO} 2.
}%progstyle
}
}
\caption{Multi-step Backup Policy-iteration with local goal-oriented behavior ($\mbox{LGO-MSBPI}$).}
\label{algorithmGoals1}
\end{small}
\end{figure}

\subsection{Convergence of the Algorithm and Its Complexity} 

\begin{lemma}
The algorithm $\mbox{LGO-MSBPI}$ in Figure~\ref{algorithmGoals1} converges to the optimal solution.
\end{lemma}
\begin{proof}
The set of global states $S$ and the set of local goal states $\hat{G}_i\subseteq S$ are finite.
The horizon $T$ is also finite. Therefore, step 3 in the algorithm will terminate.
Like the classical policy-iteration algorithm, the $\mbox{LGO-MSBPI}$ algorithm also will converge after a finite
numbers of calls to step 3 where the policy can only improve its value from one iteration
to another.
\hfill $\Box$
\end{proof}

\begin{lemma}
The complexity of computing the optimal mechanism based on local goal-oriented behavior following the $\mbox{LGO-MSBPI}$ algorithm is polynomial in the size of the state space.
\end{lemma}
\begin{proof}
Step 2 of the $\mbox{LGO-MSBPI}$ algorithm can be computed with dynamic programming in
polynomial  time
(the value of a state is computed in a backwards manner from a finite horizon $T$).
The complexity of improving a policy in Step 3 is polynomial in  the time, number of states and number
of goal states, i.e., $O(T^2|S||\hat{G}|)$.
In the worst case, every component of a global state can be a local goal state.
However, in other cases, $|\hat{G}_i|$ can be much smaller than $|S_i|$ when $\hat{G}_i$ is a strict subset 
of $S_i$, decreasing even more the running time of the algorithm. %complexity of the algorithm.
\hfill $\Box$
\end{proof}

\subsection{Experiments - Goal-oriented Options}
\label{ProdExp}

We illustrate the $\mbox{LGO-MSBPI}$ decomposition mechanism in a production control scenario.
We assume that there are two machines, which can control the production of boxes and cereals:
machine $M_1$ can produce two types of boxes $a$ or $b$. The amount of boxes of type $a$ produced
by this machine is denoted by $B_a$ ($B_b$ represents the amount of boxes of type $b$ produced respectively).
Machine $M_2$ can produce two kinds of cereals $a$ and $b$. $C_a$ (and $C_b$ respectively) denotes
the number of bags of cereals of type $a$ (we assume that one bag of cereals is sold in one box of the
same type). The boxes differ in their presentation so that
boxes of type $a$ advertise their content of type $a$ and boxes of type $b$ advertise their content
of type $b$. We assume that at each discrete time $t$, machine $M_1$ may produce one box or no boxes at all, 
and the other machine may produce one bag of cereals or may produce no cereal at all.
This production process is stochastic in the sense that the machines are not perfect: with probability $P_{M_1}$,
machine one succeeds in producing the intended box (either $a$ or $b$) and with probability $1-P_{M_1}$, 
the machine does not produce any box in that particular time unit. Similarly, we assume $P_{M_2}$ 
expresses the probability of machine two producing one bag of cereals of type $a$ or $b$ that 
is required for selling in one box. In this example, %the positive joint reward ($\mbox{JR}$) 
the reward attained by the system at $T$ is equal to the number of products ready for
sale, i.e., $min\{B_a,C_a\}+min\{B_b,C_b\}$. A product that can be sold is 
composed of one box together with one bag of cereals corresponding to the type
advertised in this box.

A goal-oriented option is given by the number of products that each machine should produce.
Therefore, an option $opt_i$ in this scenario is described by a pair of 
numbers $(X_a,X_b)$ (when $i$ is machine one then $X$ refers to boxes and when
$i$ is machine two, $X$ refers to bags of cereals). That is, machine $i$ is instructed 
to produce $X_a$ items of type $a$, followed by $X_b$ items of 
type $b$, followed by  $X_a$ items of type $a$ and so forth until either the time limit is over or 
anyone of the machines decides to communicate. Once the machines exchange information, the global state is 
revealed, i.e., the current number of boxes and cereals produced so far is known.
Given a set of goal-oriented options, the $\mbox{LGO-MSBPI}$ algorithm returned
the optimal joint policy of action and communication that solves this problem.
We counted the time units that it takes to produce the boxes with cereals.
We compared the locally goal oriented multi-step backup policy iteration algorithm ($\mbox{LGO-MSBPI}$) with two other
approaches: 1) the {\em Ideal} case when machines can exchange information about their state of production
at each time and at no cost. This is an idealized case, since in reality exchanging
information does incur some cost, for example changing the setting of a machine takes
valuable time and 2) the {\em Always Communicate} ad-hoc case, when the machines exchange information
at each time step and they also incur a cost when they do it.
Tables~\ref{AvgUC01}, \ref{AvgUC1}, and \ref{AvgUC10} present the average utility obtained by the 
production system when the cost of communication was set to $-0.1$, $-1$ and 
$-10$ respectively, the cost of a domain action was set to $-1$ and the 
joint utility was averaged over 1000 experiments.
A state is represented by the tuple $(B_a,B_b,C_a,C_b)$.
The initial state was set to (0,0,0,8), there were no boxes produced and there were
already 8 bags of cereals of type $B$.
The finite horizon $T$ was set to 10. The set of goal-oriented options 
$(X_a,X_b)$ tested included (0,1),(1,4),(2,3),(1,1),(3,2),(4,1) and (1,0).
%%\footnote{The machines do not need to communicate if they 
%%start from (0,0,0,0). The options chosen then are (0,1) and $k=T$.}

%The tex file with the tables/results is in /nfs/coord/u0/clag/MAMDP/SMDP2ndJAIR04/PRODEXPS0008.tex
\begin{table}[hbtp]
\begin{small}
\centerline
{
\begin{tabular}{|l|l|l|l|l|l}
\hline
&\multicolumn{3}{c|}{Average Utility} \\ \cline{2-4}\multicolumn{1}{|c|}{$P_{M_1},P_{M_2}$}
&\multicolumn{1}{|c|}{Ideal $C_\Sigma=0$} &\multicolumn{1}{c|}{Always Communicate} &\multicolumn{1}{c|}{LGO-MSBPI} \\ \hline \hline
$0.2,0.2$ & -17.012 & -18.017& -17.7949\\ \hline 
$0.2,0.8$ & -16.999 & -17.94 & -18.0026  \\ \hline
$0.8,0.8$ & -11.003 & -12.01 & -12.446 \\ \hline 
\end{tabular}
}
\caption{$C_{\Sigma}=-0.10$, $R_a=-1.0$.}
\label{AvgUC01}
\end{small}
\end{table}

\begin{table}[hbtp]
\begin{small}
\centerline{
\begin{tabular}{|l|l|l|l|l|l}
\hline
&\multicolumn{3}{c|}{Average Utility} \\ \cline{2-4}\multicolumn{1}{|c|}{$P_{M_1},P_{M_2}$}
& \multicolumn{1}{|c|}{Ideal $C_\Sigma=0$} &\multicolumn{1}{c|}{Always Communicate} &\multicolumn{1}{c|}{LGO-MSBPI} \\ \hline \hline
$0.2,0.2$ & -17.012 & -26.99  & -19.584   \\ \hline 
$0.2,0.8$ & -16.999 & -26.985 & -25.294 \\ \hline
$0.8,0.8$ & -11.003 & -20.995 & -17.908  \\ \hline 
\end{tabular}
}
\caption{$C_{\Sigma}=-1.0$, $R_a=-1.0$.}
\label{AvgUC1}
\end{small}
\end{table}

\begin{table}[hbtp]
\begin{small}
\centerline{
\begin{tabular}{|l|l|l|l|l|l}
\hline
&\multicolumn{3}{c|}{Average Utility} \\ \cline{2-4}\multicolumn{1}{|c|}{$P_{M_1},P_{M_2}$}
& \multicolumn{1}{|c|}{Ideal $C_\Sigma=0$} &\multicolumn{1}{c|}{Always Communicate} &\multicolumn{1}{c|}{LGO-MSBPI} \\ \hline \hline
$0.2,0.2$ & -17.012 & -117      & -17.262 \\ \hline 
$0.2,0.8$ & -16.999 & -117.028  & -87.27 \\ \hline
$0.8,0.8$ & -11.003 & -110.961  & -81.798  \\ \hline 
\end{tabular}
}
\caption{$C_{\Sigma}=-10.0$, $R_a=-1.0$.}
\label{AvgUC10}
\end{small}
\end{table}

The $\mbox{LGO-MSBPI}$ algorithm computed a mechanism that resulted in three
products on average when the uncertainty of at least one machine was set to 
$0.2$ and 1000 tests were run, each for ten time units. The number of products increased on average 
between 8 to 9 products when the machines succeeded $80\%$ of the cases. These 
numbers of products were always attained either when the 
decomposition mechanism was implemented or when the ad-hoc approaches were tested.
Ideal or Always Communicate algorithms only differ with respect to the cost of communication,
and they do not differ in the actual policies of action.
Although the machines incur a higher cost when the mechanism is applied compared to the ideal case 
(due to the cost of communication), the number of final products ready to sell were almost the same 
amount. That is, it will take some more time in order to produce the right amount 
of products when the policies implemented are those computed by the locally goal oriented 
multi-step backup policy iteration algorithm. The cost of communication in this scenario
can capture the cost of changing the setting of one machine from one production program to 
another. Therefore, our result is significant when this cost of communication is very high 
compared to the time that the whole process takes. The decomposition mechanism finds
what times are most beneficial to synchronize information when constant communication
is not feasible nor desirable due to its high cost.

\subsection{Generalization of the $\mbox{LGO-MSBPI}$ Algorithm}
\label{LGO-nonindependent}

The mechanism approach assumes that agents can operate independent of each
other for some period of time. However, if the decentralized process has some 
kind of dependency in its observations or transitions, this assumption will be 
violated, i.e., the plans to reach the local goals can interfere with each 
other (the local goals may not be compatible). The $\mbox{LGO-MSBPI}$ 
algorithm presented in this paper can be applied to Dec-MDPs when their
transitions and observations are not assumed to be independent.
In this section, we bound the error in the utilities of the options computed 
by the $\mbox{LGO-MSBPI}$ algorithms when such dependencies do exist.
We define $\Delta\!-\!independent$ decentralized processes
to refer to {\em nearly-independent} processes whose dependency can be quantified
by the cost of their marginal interactions.

\begin{definition}[$\Delta\!-\!independent$ Process]
\label{Deltaindependence}
Let $\overline{C}_{A_i}(s \rightarrow \hat{g}_k|\hat{g}_j)$ be the expected cost incurred by
agent $i$  when following its optimal local policy to reach local goal state $\hat{g}_k$ 
from state $s$, while the other agent is following its optimal policy to reach 
$\hat{g}_j$.
A decentralized control process is $\Delta\!-\!independent$ if $\Delta=\max\{\Delta_1,\Delta_2\}$, where
$\Delta_1$ and $\Delta_2$ are defined as follows: $\forall \hat{g}_1,\hat{g}'_1 \in \hat{G}_1 \in S_1$, $\hat{g}_2,\hat{g}'_2 \in \hat{G}_2 \in S_2$ and $s\in S$:
\[\Delta_1 = \max_s\{\max_{\hat{g}_1}\{\max_{\hat{g}_2,\hat{g}'_2}\{\overline{C}_{A_1}(s^0\rightarrow \hat{g}_1|\hat{g}'_2)-\overline{C}_{A_1}(s^0\rightarrow \hat{g}_1|\hat{g}_2)\}\}\}\]
\[\Delta_2 = \max_s\{\max_{\hat{g}_2}\{\max_{\hat{g}_1,\hat{g}'_1}\{\overline{C}_{A_2}(s^0\rightarrow \hat{g}_2|\hat{g}'_1)-\overline{C}_{A_2}(s^0 \rightarrow \hat{g}_2|\hat{g}_1)\}\}\}\]

That is, $\Delta$ is the maximal difference in cost that an agent may incur when trying to
reach one local goal state that interferes with any other possible local goal being 
reached by the other agent.
\end{definition}

The computation of the cost function $\overline{C}_{A_i}(s \rightarrow \hat{g}_k|\hat{g}_j)$ 
is domain dependent. We do not address the issue of how to compute this cost but we provide the 
condition. The individual costs of one agent can be affected by the interference that exists between some 
pair of local goals. For example, assume a 2D grid scenario: one agent can move only in four 
directions (north, south, east and west) and needs to reach location (9,9) from (0,0). The second 
agent is able of moving and also of collecting rocks and blocking squares in the grid.
Assuming that the second agent is assigned the task of blocking all the squares
in even rows, then the first agent's solution to its task is constrained
by the squares that are free to cross. In this case, agent one's cost to reach
(9,9) depends on the path it will choose that depends very strongly on the 
state of the grid resulting from the second agent's actions.

The $\Delta$ value denotes the amount of interference that might occur between the agents' 
locally goal-oriented behaviors. When the Dec-MDP has independent transitions and observations, 
the value of $\Delta$ is zero. The $\mbox{LGO-MSBPI}$ algorithm proposed in this paper computes 
the mechanism for each global state as a mapping from states to pairs of local goal states 
ignoring the potential interference. Therefore, the difference between the actual cost that will 
be incurred by the options found by the algorithm and the optimal options can be at most $\Delta$. Since 
the mechanism is applied for each global state for $T$ time steps and this loss in cost can 
occur in the worst case for both agents, the algorithm presented here is 
$2T\Delta\!-\!optimal$ in the general case.

\section{A Myopic-greedy Approach to Direct Communication}
\label{greedy}

In some cases, it is reasonable to assume that {\em single}-agent 
behaviors are already  known and fixed, ahead of time for any possible global
state. For example, this may occur in settings 
where individual agents are designed ahead of the coordination time (e.g., agents in a 
manufacturing line can represent machines, which are built specifically to implement certain procedures).
To achieve coordination, though, some additional method may be needed to synchronize these
individual behaviors. In this section, we present how to apply the communication-based decomposition
approach to compute the policy of communication that will synchronize the {\em given} goal-oriented options.
We take a myopic-greedy approach that runs in polynomial-time: i.e., each time 
an agent makes a decision, it chooses the action with maximal expected 
accumulated reward assuming that agents are only able to communicate once along the whole process.
Notice that the $\mbox{LGO-MSBPI}$ was more general in the sense that it also computed {\em what} local goals 
should be pursued by each agent together with the communication policy that 
synchronizes their individual behaviors. 
Here, each time the agents exchange information, the mechanism is applied 
inducing
two individual behaviors (chosen from the given mapping from states to individual behaviors). 
The given optimal policies of action (with no communication actions) are denoted
$\delta_1^{A*}$ and $\delta_2^{A*}$ respectively. 

The expected global reward of the system, given that the agents {\bf do not 
communicate at all} and each follows its corresponding optimal policy $\delta_i^{A*}$
is given by the value of the initial state $s^0$: 
$\Theta^\delta_{nc}(s^0,\delta_1^{A*},\delta_2^{A*})$. This value can be 
computed by summing over all possible next states and computing the probability
of each agent reaching it, the reward obtained then and the recursive value
computed for the next states.
\begin{small}
\[\Theta^\delta_{nc}(s^0,\delta_1^{A*},\delta_2^{A*})=\]
\[\sum_{(s'_1,s'_2)}P_1(s'_1|s^0_1,\delta_1^{A*}(s^0_1))\cdot P_2(s'_2|s^0_2,\delta_2^{A*}(s^0_2))(R(s^0,\delta_1^{A*}(s^0_1),\delta_2^{A*}(s^0_2),s')+\Theta^\delta_{nc}(s' ,\delta_1^{A*},\delta_2^{A*}))\]
\end{small}
We denote the expected cost of the system computed by agent $i$, when
the last synchronized state is $s^0$, and {\bf when the agents communicate once} at state $s$ 
and continue without any communication,
$\Theta_{c}(s^0,s_i,\delta_1^{A*},\delta_2^{A*})$:
\begin{small}
\[\Theta_{c}(s^0,s_1,\delta_1^{A*},\delta_2^{A*})=\]
\[\sum_{s_2} \overline{P_2}(s_2|s^0_2,\delta_2^{A*})(\overline{R}(s^0,\delta_1^{A*}(s^0_1),\delta_2^{A*}(s_2^0),(s_1,s_2))+\Theta^\delta_{nc}((s_1,s_2),\delta_1^{A*},\delta_2^{A*})+C_\Sigma\cdot Flag)\]
\end{small}
Flag is zero if the agents reached the global goal state before they reached 
state $s$. The time stamp in state $s$ is denoted $t(s)$. 
$\overline{P}(s|,s^0,\delta^{A*}_1,\delta^{A*}_2)$ is the probability 
of reaching state $s$ from state $s^0$, following the given policies of action.
\begin{small}
\[\overline{P}(s'|s,\delta^{A*}_1,\delta^{A*}_2) = \left\{ \begin{array}{ll}
                      1            &\mbox{$if \ s=s'$}\\
                      P(s'|s,\delta^{A*}_1(s_1),\delta^{A*}_2(s_2)) &\mbox{$if \ t(s')=t(s)+1$}\\
                      0           &\mbox{$if \ t(s') < t(s) + 1$} \\
                      \sum_{s^{''}\in S} \overline{P}(s'|s^{''},\delta^{A*}_1,\delta^{A*}_2)\cdot P(s^{''}|s,\delta^{A*}_1,\delta^{A*}_2)       &\mbox{$otherwise$}      
                     \end{array}
            \right. \]
\end{small}
Similarly, $\overline{P_1}$ ($\overline{P_2}$) can be defined for the 
probability of reaching $s'_1$ ($s'_2$), given agent $1$ ($2$)'s current 
partial view $s_1$ ($s_2$) and its policy of action $\delta_1^{A*}$ ($\delta_2^{A*}$).
The accumulated reward attained while the agents move from state $s^0$ to state $s$ is given as follows:
\begin{small}
\[\overline{R}(s^0,\delta^{A*}_1,\delta^{A*}_2,s) = \left\{ \begin{array}{ll}
                      R(s^0,\delta^{A*}_1(s^0_1),\delta^{A*}_2(s^0_2),s) &\mbox{$if \ t(s)=t(s^0)+1$}\\
\ &\mbox{$if \ t(s) > t(s^0) + 1$} \\
                      \sum_{s^{''}}\overline{P}(s^{''}|\delta^{A*}_1,\delta^{A*}_2,s^0)\cdot P(s|\delta^{A*}_1,\delta^{A*}_2,s^{''})\cdot\\
(\overline{R}(s^0,\delta^{A*}_1,\delta^{A*}_2,s^{''})+R(s^{''},\delta^{A*}_1(s_1^{''}),\delta^{A*}_2(s_2^{''}),s)) 
                     \end{array}
            \right. \]
\end{small}
At each state, each agent decides whether to communicate its partial view or
not based on whether the expected cost from following the policies of action,
and having communicated is larger or smaller than the expected cost from
following these policies of action and not having communicated.

\begin{lemma}
\label{greedyP}
Deciding a Dec-MDP-Com with the myopic-greedy approach to direct 
communication is in the P class.
\end{lemma}

\begin{proof}
Each agent executes its known policy $\delta^{A*}_i$ when the mechanism is applied. If local
goals are provided instead of actual policies, finding the optimal single-agent policies 
that reach those goal states can be done in polynomial time. 
The complexity of finding the communication policy is the same as 
dynamic programming (based on the formulas above), therefore computing the policy of communication is also in P.
There are $|S|$ states for which $\Theta^\delta_{nc}$ and $\Theta_c$ need to be computed, and each one of these 
formulas can be solved in time polynomial in $|S|$.
\hfill $\Box$
\end{proof}

In previous work, we have also studied the set of {\em monotonic} goal-oriented Dec-MDPs, 
for which we provide an algorithm that finds the optimal policy of communication assuming a
set of individual behaviors is provided~\cite{Goldman04e}.

\subsection{Meeting Under Uncertainty Example}
\label{example}

We present empirical results obtained when the myopic-greedy approach was
applied to the Meeting under Uncertainty example\footnote{Some of the empirical results in this section were described first by~\citeauthor{Goldman03a}~\citeyear{Goldman03a}.}. 
The testbed we consider is a sample problem of a Dec-MDP-Com involving two
agents that have to meet at some location as early as possible.  This scenario is also known
as the gathering problem in robotics~\cite{Suzuki99}. The
environment is represented by a 2D grid with discrete locations.
In this example, any global state that can be occupied by both agents is considered a global 
goal state. The set of control actions includes moving North, South, East and West, and 
staying at the same location. Each agent's partial view (which is locally fully-observable) 
corresponds to the agent's location coordinates. The observations and the transitions are 
independent. The outcomes of the agents' actions are uncertain: that is,
with probability $P_i$, agent $i$ arrives at the desired
location after having taken a move action, but with probability $1-P_i$ the
agent remains at the same location. Due to this uncertainty in the effects of
the agents' actions, it is not clear that setting a predetermined meeting point
is the best strategy for designing these agents. Agents may be able to meet faster if they
change their meeting place after realizing their actual locations. This can be achieved by
exchanging information on the locations of the agents, that otherwise are
not observable. We showed that exchanging the last
observation guarantees optimality in a Dec-MDP-Com process with constant message costs~\cite{Goldman04c}.
In the example tested, the messages exchanged correspond to the agents' own observations, i.e., 
their location coordinates. 

We have implemented the locally goal-oriented mechanism that assigns a single local goal to
each agent at each synchronized state. These local goals were chosen as
the location in the middle of the shortest Manhattan path between the agents' locations (this distance is 
revealed when information is exchanged). 

Intuitively, it is desirable for a mechanism to set a meeting place in the middle
of the shortest Manhattan path that connects the two agents because in the
absence of communication, the cost to meet at that point is minimal.
This can be shown by computing the expected time to meet, $\Theta_{nc}$,
for any pair of possible distances between the two agents and any location in
the grid, when no communication is possible. 
%The minimal value is attained when these distances are equal. 
To simplify the exposition, we use a function
that takes advantage of the specific characteristics of the example.
The notation is as follows: agent $1$ is at distance $d_1$ from the meeting
location, agent $2$ is at distance $d_2$ from that location, the system
incurs a cost of one at each time period if the agents have not met yet.
If both agents are at the meeting location, the %joint 
expected time to meet
is zero, $\Theta_{nc}(0,0)=0$. If only agent $2$ is at the meeting location, but
agent $1$ has not reached that location yet, then the %joint
expected time to meet is given by
\begin{small}
\[\Theta_{nc}(d_1,0)=P_1\cdot (-1+\Theta_{nc}(d_1\!-\!1,0))+(1\!-\!P_1)\cdot (-1+\Theta_{nc}(d_1,0))=\]
\[\ \ \ \ =P_1\cdot \Theta_{nc}(d_1\!-\!1,0))+(1\!-\!P_1)\cdot \Theta_{nc}(d_1,0))-1\]
\end{small}
That is, with probability $P_1$ agent $1$ succeeds in decreasing its distance to
the meeting location by one, and with probability $1-P_1$ it fails and remains
at the same location. Recursively, we can compute the remaining %joint
expected time to meet with the updated parameters. Similarly for agent $2$:
$\Theta_{nc}(0,d_2)=P_2\cdot (-1+\Theta_{nc}(0,d_2\!-\!1))+(1\!-\!P_2)\cdot (-1\!+\!\Theta_{nc}(0,d_2)).$
If none of the agents has reached the meeting place yet, then there are
four different cases in which either both, only one, or none succeeded in
moving in the right direction and either or not decreased their distances
to the meeting location respectively:
\begin{small}
\[\Theta_{nc}(d_1,d_2)= P_1 \cdot P_2\cdot (-1+\Theta_{nc}(d_1\!-\!1,d_2\!-\!1))+P_1\cdot (1\!-\!P_2)\cdot (-1+\Theta_{nc}(d_1\!-\!1,d_2))+\]
\[\ \ \ \ + (1\!-\!P_1)\cdot P_2\cdot (-1+\Theta_{nc}(d_1,d_2\!-\!1))+(1\!-\!P_1)\cdot (1\!-\!P_2)\cdot (-1+\Theta_{nc}(d_1,d_2))=\]
\[\ \ \ \ = P_1 \cdot P_2\cdot \Theta_{nc}(d_1\!-\!1,d_2\!-\!1)+P_1\cdot (1\!-\!P_2)\cdot \Theta_{nc}(d_1\!-\!1,d_2)+ (1\!-\!P_1)\cdot P_2\cdot \Theta_{nc}(d_1,d_2\!-\!1)+\]
\[+(1\!-\!P_1)\cdot (1\!-\!P_2)\cdot \Theta_{nc}(d_1,d_2)-1\]
\end{small}
The value of $\Theta_{nc}(d_1,d_2)$ was computed for all possible distances $d_1$ and $d_2$ in
a 2D grid of size $10\times 10$. The minimal expected time to meet
was obtained when $d_1=d_2=9$ and the expected cost was $-12.16$.

In summary, approximating the optimal solution to the Meeting under Uncertainty
example when direct communication is possible and the mechanism applied is the
one described above will unfold as follows: At time $t_0$, the initial
state of the system $s^0$ is fully observable by both agents. The agents set a
meeting point in the middle of a Manhattan path that connects them.
Denote by $d_0$ the distance between the agents
at $t_0$ and $g_{t_0}=(g^1_{t_0},g^2_{t_0})$
the goal state set at $t_0$. Each one of the agents can move optimally towards
its corresponding component of $g_{t_0}$.
Each agent moves independently in the environment because the transitions and
observations are independent. Each time $t$, when the policy of
communication instructs an agent to initiate exchange of information, the
current Manhattan distance between the agents $d_t$ is revealed to both.
Then, the mechanism is applied, setting a possibly new goal state
$g_t$, which decomposes into two components one for each agent. This
goal state $g_t$ is in the middle of the Manhattan path that connects the
agents with length $d_t$ revealed through communication.

\subsection{Experiments - Myopic-greedy Approach}

In the following experiments, we assumed that the transition probabilities $P_1$ and $P_2$
are equal. These uncertainties were specified by the parameter $P_u$. The mechanism that is applied 
whenever the agents communicate at time $t$ results in each agent adopting a 
local goal state, that is set at the location in the middle of the Manhattan 
path connecting the agents (the Manhattan distance between the agents is 
revealed at time $t$).
We compare the joint utility attained by the system in the following four different 
scenarios:
\begin{enumerate}
\cutwspace
\item No-Communication --- The meeting point is fixed at time $t_0$ and 
remains fixed along the simulation. It is located in the middle of the 
Manhattan path that connects between the agents, known at time $t_0$. Each 
agent follows its optimal policy of action to this location without communicating.
\item Ideal --- This case assumes that the agents can communicate freely ($C_\Sigma=0$)
at every time step resulting in the highest global utility that both agents can attain.
Notice, though, that this is not the optimal solution we are looking for, 
because we do assume that communication is not free. Nevertheless, the 
difference in the utility obtained in these first two cases shed light on the 
trade-off that can be achieved by implementing non-free communication policies.
\item Communicate SubGoals --- A heuristic solution to the problem, 
which assumes that the agents have a notion of sub-goals. They notify each 
other when these sub-goals are achieved, eventually leading the agents to meet.
\item Myopic-greedy Approach --- Agents act myopically optimizing the choice of when
to send a message, assuming no additional communication is possible.
For each possible distance between the agents, a policy of communication is
computed such that it stipulates the best time to send that message.
By iterating on this policy agents are able to communicate more than once
and thus approximate the optimal solution to the decentralized control
problem with direct communication. The agents continue moving until they meet.
\end{enumerate}

%Experiments results in tablesStartt1.AAMAS03.ps.gz from opine/TALKS

The solution to the No-Communication case can be solved analytically for the 
Meeting under Uncertainty example, by computing the expected cost
$\Theta_{nc}(d_1,d_2)$ incurred by two agents 
located at distances $d_1$ and $d_2$ respectively from the goal state at time 
$t_0$ (the complete mathematical solution appears in Section~\ref{example}).
In the Ideal case, a set of 1000 experiments was run with cost of communication set to 
zero. Agents communicate their locations 
at every time instance, and update the location of the meeting place 
accordingly. Agents move optimally to the last synchronized meeting location.

For the third case tested (Communicate SubGoals) a sub-goal was defined by 
the cells of the grid with distance equal to $p\cdot d/2$ and with center located
at $d/2$ from each one of the agents. %from the fixed current meeting point. 
$p$ is a parameter of the problem that determines the radius 
of the circle that will be considered a sub-goal. Each time an agent reaches
a cell inside the area defined as a sub-goal, it initiates exchange of
information (therefore, $p$  induces the communication strategy). 
$d$ expresses the Manhattan distance between the two 
agents, this value is accurate only when the agents synchronize their 
knowledge. That is, at time $t_0$ the agents determine the first sub-goal as the
area bounded by a radius of $p\cdot d_0/2$ and, which center is located at $d_0/2$ 
from each one of the agents. Each time $t$ that the agents synchronize their 
information through communication, a new sub-goal is determined at $p\cdot d_{t}/2$.
Figure~\ref{CommSubG} shows how new sub-goals are set when the agents transmit 
their actual location once they reached a sub-goal area. The meeting point
is dynamically set at the center of the sub-goal area.

\begin{figure}[t]
% \centerline{\psfig{figure=CommSubG.eps,height=4cm}}
\centerline{\includegraphics[height=4cm]{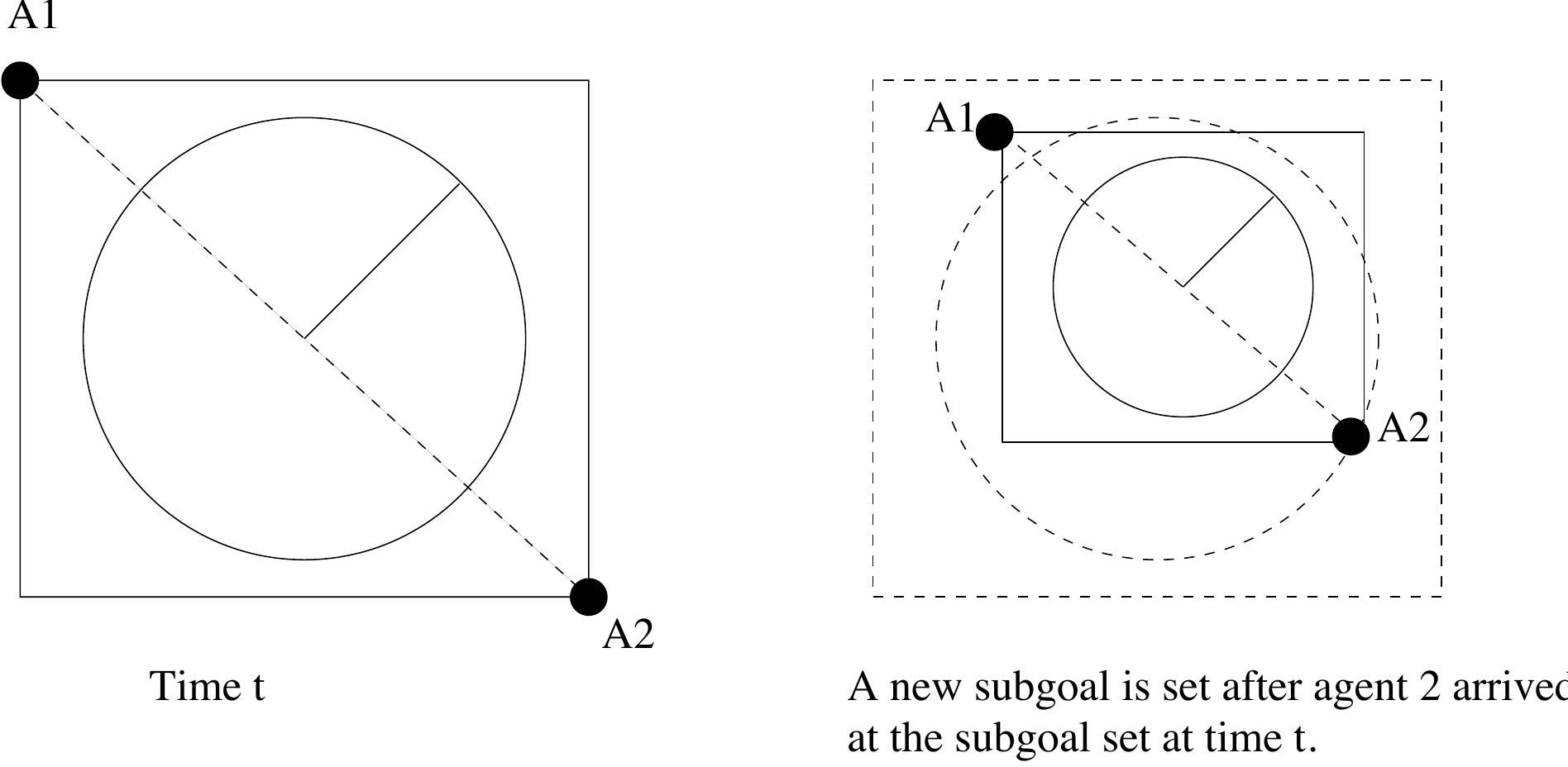}}
\caption{Goal decomposition into sub-goal areas.}
\label{CommSubG}
\end{figure}

Experiments were run for the Communicate SubGoals case for different 
uncertainty values, values of the parameter $p$ and costs of communication (for each
case, 1000 experiments were run and averaged). 
These results show that agents can obtain
higher utility by adjusting the meeting point dynamically rather than having 
set one fixed meeting point. Agents can synchronize their knowledge and thus
they can set a new meeting location instead of acting as two independent
MDPs that do not communicate and move towards a fixed meeting point (see 
Figure \ref{AvgUPr08}). Nevertheless, for certain values of $p$, the joint 
utility of the agents is actually smaller than the joint utility achieved in 
the No-Communication case (2 MDPs in the figure). 
This points out the need to empirically tune up the parameters required by the 
heuristic. 

\begin{figure}[t]
%\centerline{\psfig{figure=Fig5Shorter.porta.eps,height=8cm,width=8cm}}
\centerline{\includegraphics[height=7.5cm]{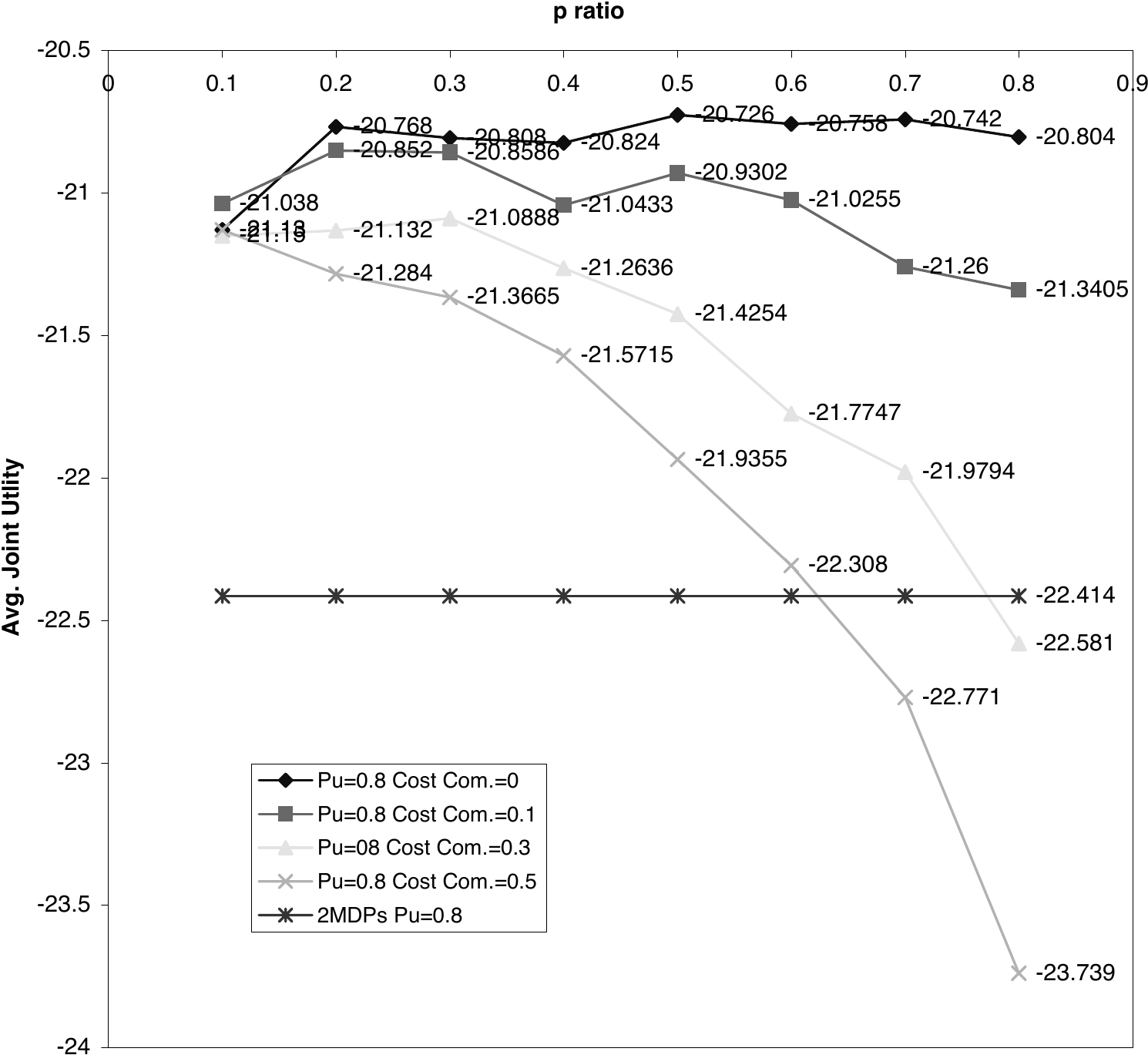}}
\caption{The average joint utility obtained when sub-goals are communicated.}
\label{AvgUPr08}
\end{figure}

In the Myopic-greedy case, we design the agents to optimize the time when
they should send a message, assuming that they can communicate only once.
At the off-line planning stage, the agents compute their expected joint 
cost to meet for any possible state of the system ($s^0$) and time t (included in the local 
state $s_i$), $\Theta_c(s^0,s_i,\delta_1^{A*},\delta_2^{A*})$. The global states revealed through 
communication correspond to the possible distances between the agents. Each time the agents
get synchronized, the mechanism is applied assigning local goals and instructing the agents
to follow the optimal local policies to achieve them.
In the Meeting under Uncertainty scenario we study, $\Theta_c$ 
is the expected joint cost incurred by taking control actions during $t$ time steps, 
communicating then at time $t+1$ if the agents have not met so far, and continuing with the 
optimal policy of control actions without communicating towards the goal state (the meeting
location agreed upon at $t+1$) at an expected cost of $\Theta_{nc}(d_1,d_2)$ as computed for 
the No-Communication case. When the agents meet before the $t$ time steps have elapsed, they only
incur a cost for the time they act before they met.

At each time $t$, each one of the agents knows a meeting location, that is
the goal location computed from the last exchange of information. Consequently,
each agent moves optimally towards this goal state. In addition, the 
myopic-greedy policy of communication is found by computing the earliest time 
$t$, 
%for which $\Theta_c(d,t) < \Theta_{nc}(d_1,d_2)$, that is, what is the 
for which $\Theta_c(d_1+d_2,s_1,\delta_1^{A*},\delta_2^{A*}) < \Theta_{nc}(d_1,d_2)$, that is, what is the 
best time to communicate such that the expected cost to meet is the least. The 
myopic-greedy policy of communication is a vector that states the time to communicate for each 
possible distance between the agents.

Tables~\ref{notationC01R1}, \ref{notationC1R1}, and \ref{notationC10R1} present
the myopic-greedy communication policies computed for the Meeting under 
Uncertainty problem with $P_u$ values taken from 
$\{0.2, 0.4, 0.6, 0.8\}$.  The cost of taking a control action is $R_a=-1.0$ 
and the costs of communicating $C_\Sigma$ tested were $\{-0.1, -1.0,-10.0\}$. 
%13.3.05 The resulting policies of communication are presented in Appendix~\ref{GreedyPoliciesComm}.
Each row corresponds to a configuration tested with different state-transition uncertainties.
Each column corresponds to a synchronized state, given by the possible
Manhattan distance between the agents moving in a  2D grid of size 10x10.
Given a certain value for $P_u$ and a certain global distance, each agent interprets the value
in any entry as the next time to communicate its position.
Time is reset to zero when the agents exchange information.
As long as the distance between the agents is larger and the communication cost increases,
the policy instructs the agents to communicate later, i.e., the agents should keep operating
until the information exchanged will have a better effect on the rescheduling of the
meeting place.

\begin{table}[hbtp]
\begin{small}
\centerline{
\begin{tabular}{|l|l|l|l|l|l|l|l|l|l|l|l|l|l|l|l|l|l|l|}
\hline
&\multicolumn{18}{c|}{$d_0$=distance between agents when last synchronized, $g$ located at $d_0/2$} \\ \cline{2-19}\multicolumn{1}{|c|}{$P_u$}
& \multicolumn{1}{|c|}{1}&\multicolumn{1}{c|}{2} &\multicolumn{1}{c|}{3} &\multicolumn{1}{c|}{4} &\multicolumn{1}{c|}{5} &\multicolumn{1}{c|}{6} &\multicolumn{1}{c|} {7} &\multicolumn{1}{c|}{8} &\multicolumn{1}{c|}{9} &\multicolumn{1}{c|}{10}&\multicolumn{1}{c|}{11}&\multicolumn{1}{c|}{12}&\multicolumn{1}{c|}{13}&\multicolumn{1}{c|}{14}&\multicolumn{1}{c|}{15}&\multicolumn{1}{c|}{16}&\multicolumn{1}{c|}{17}&\multicolumn{1}{c|}{18} \\ \hline \hline
$0.2$ &2 &3 &2 &3 &2 &3 &2 &3 &2 &3 &2 &3 &2 &3 &2 &3 &2 &3 \\ \hline
$0.4$ &2 &2 &2 &3 &2 &3 &2 &3 &2 &3 &2 &3 &2 &3 &2 &3 &2 &3 \\ \hline
$0.6$ &2 &2 &2 &3 &2 &3 &2 &3 &2 &3 &2 &3 &2 &3 &2 &3 &2 &3 \\ \hline
$0.8$ &2 &2 &2 &3 &2 &4 &2 &4 &2 &4 &2 &4 &2 &4 &2 &4 &2 &4 \\ \hline
\end{tabular}
}
%\caption{The optimal time when agents should communicate if they can only
%communicate once, where $C_{\Sigma}=-0.1$,$R_a=-1.0$.}
\caption{Myopic-greedy policy of communication: $C_{\Sigma}=-0.1$, $R_a=-1.0$.}
\label{notationC01R1}
\end{small}
\end{table}

\begin{table}[hbtp]
\begin{small}
\centerline{
\begin{tabular}{|l|l|l|l|l|l|l|l|l|l|l|l|l|l|l|l|l|l|l|}
\hline
&\multicolumn{18}{c|}{$d_0$=distance between agents when last synchronized, $g$ located at $d_0/2$} \\ \cline{2-19}\multicolumn{1}{|c|}{$P_u$}
& \multicolumn{1}{|c|}{1}&\multicolumn{1}{c|}{2} &\multicolumn{1}{c|}{3} &\multicolumn{1}{c|}{4} &\multicolumn{1}{c|}{5} &\multicolumn{1}{c|}{6} &\multicolumn{1}{c|} {7} &\multicolumn{1}{c|}{8} &\multicolumn{1}{c|}{9} &\multicolumn{1}{c|}{10}&\multicolumn{1}{c|}{11}&\multicolumn{1}{c|}{12}&\multicolumn{1}{c|}{13}&\multicolumn{1}{c|}{14}&\multicolumn{1}{c|}{15}&\multicolumn{1}{c|}{16}&\multicolumn{1}{c|}{17}&\multicolumn{1}{c|}{18} \\ \hline \hline
$0.2$ &3 &4 &3 &5 &3 &6 &4 &7 &4 &7 &5 &7 &5 &8 &5 &8 &6 &9 \\ \hline
$0.4$ &2 &3 &3 &4 &4 &5 &4 &6 &5 &7 &5 &7 &6 &8 &6 &8 &7 &9 \\ \hline
$0.6$ &2 &2 &3 &4 &4 &5 &5 &6 &6 &7 &6 &8 &7 &8 &7 &9 &8 &10 \\ \hline
$0.8$ &2 &2 &3 &3 &4 &4 &5 &5 &6 &6 &7 &7 &8 &8 &9 &9 &10 &10 \\ \hline
\end{tabular}
}
%\caption{The optimal time when agents should communicate if they can only
%communicate once, where $C_{\Sigma}=-1.0$,$R_a=-1.0$.}
\caption{Myopic-greedy policy of communication: $C_{\Sigma}=-1.0$, $R_a=-1.0$.}
\label{notationC1R1}
\end{small}
\end{table}

\begin{table}[hbtp]
\begin{small}
\centerline{
\begin{tabular}{|l|l|l|l|l|l|l|l|l|l|l|l|l|l|l|l|l|l|l|}
\hline
&\multicolumn{18}{c|}{$d_0$=distance between agents when last synchronized, $g$ located at $d_0/2$} \\ \cline{2-19}\multicolumn{1}{|c|}{$P_u$}
& \multicolumn{1}{|c|}{1}&\multicolumn{1}{c|}{2} &\multicolumn{1}{c|}{3} &\multicolumn{1}{c|}{4} &\multicolumn{1}{c|}{5} &\multicolumn{1}{c|}{6} &\multicolumn{1}{c|} {7} &\multicolumn{1}{c|}{8} &\multicolumn{1}{c|}{9} &\multicolumn{1}{c|}{10}&\multicolumn{1}{c|}{11}&\multicolumn{1}{c|}{12}&\multicolumn{1}{c|}{13}&\multicolumn{1}{c|}{14}&\multicolumn{1}{c|}{15}&\multicolumn{1}{c|}{16}&\multicolumn{1}{c|}{17}&\multicolumn{1}{c|}{18} \\ \hline \hline
$0.2$ &9 &9 &11 &13 &14 &17 &18 &20 &21 &23 &25 &27 &28 &30 &32 &34 &35 &37 \\ \hline
$0.4$ &5 &6 &7 &8 &9 &10 &11 &12 &13 &14 &15 &16 &17 &18 &19 &20 &21 &22 \\ \hline
$0.6$ &4 &4 &5 &6 &6 &7 &8 &9 &9 &10 &11 &12 &12 &13 &14 &15 &15 &16 \\ \hline
$0.8$ &3 &3 &4 &4 &5 &5 &6 &7 &7 &8 &8 &9 &10 &10 &11 &11 &12 &12 \\ \hline
\end{tabular}
}
%\caption{The optimal time when agents should communicate if they can only
%communicate once, where $C_{\Sigma}=-10.0$,$R_a=-1.0$.}
\caption{Myopic-greedy policy of communication: $C_{\Sigma}=-10.0$, $R_a=-1.0$.}
\label{notationC10R1}
\end{small}
\end{table}

For the smallest cost of communication tested, it is always beneficial to communicate rather
early, no matter the uncertainty in the environment, and almost no matter
what $d_0$ is. % (the differences in time are between 2 and 4).
For larger costs of communication and a given $P_u$, the larger the distance between
the agents, the later they  will communicate (e.g., when $P_u=0.4$, $C_\Sigma=-1$ 
and $d=5$, agents should communicate at time 4, but if $C_\Sigma=-10$, they 
should communicate at time 9).
For a given $C_\Sigma$, the larger the distance between the agents is, the later
the agents will communicate (e.g., when $P_u=0.4$, $C_\Sigma=-10$ 
and $d=5$, agents should communicate at time 9, but if $d=12$, they should
communicate at time 16).
The results from averaging over 1000 runs show that for a given cost $C_\Sigma$
as long as $P_u$ decreases (the agent is more uncertain about its actions'
outcomes), the agents communicate more times.

In the 1000 experiments run, the agents exchange information about their 
actual locations at the best time that was myopically found for $d_0$ (known to both at
time $t_0$). After they communicate, they know the actual distance $d_t$, between them.
The agents follow the same myopic-greedy communication policy to 
find the next time when they should communicate if they did not meet already. This 
time is the best time found by the myopic-greedy algorithm given that 
the distance between the agents was $d_{t}$.
Iteratively, the agents approximate the optimal solution to the decentralized
control problem with direct communication by following their independent 
optimal policies of action, and the myopic-greedy policy for communication.
Results obtained from averaging the global utility attained after 1000
experiments show that these myopic-greedy agents can perform better than 
agents who communicate sub-goals (that was shown already to be more efficient than not 
communicating at all). 

\begin{table}[hbtp]
\begin{small}
\centerline{
\begin{tabular}{|l|l|l|l|l|l|}
\hline
&\multicolumn{4}{c|}{Average Joint Utility} \\ \cline{2-5}\multicolumn{1}{|c|}{$P_u$}
& \multicolumn{1}{|c|}{No-Comm.}&\multicolumn{1}{c|}{Ideal $C_\Sigma=0$} &\multicolumn{1}{c|}{SubGoals\footnotemark} &\multicolumn{1}{c|}{Myopic-Greedy} \\ \hline \hline
$0.2$ &-104.925&-62.872 &-64.7399 & -63.76 \\ \hline 
$0.4$ &-51.4522&-37.33  &-38.172  & -37.338 \\ \hline
$0.6$ &-33.4955&-26.444 &-27.232  & -26.666 \\ \hline 
$0.8$ &-24.3202&-20.584 &-20.852  & -20.704 \\ \hline
%Results after running 2 separate 2MDPs (with no cost for waiting)
%$0.2$ &-69.604 &-62.872 &-64.7399 & -63.682 \\ \hline 
%$0.4$ &-41.906 &-37.33  &-38.172  & -37.81 \\ \hline
%$0.6$ &-29.17  &-26.444 &-27.232  & -26.742 \\ \hline 
%$0.8$ &-22.414 &-20.584 &-20.852  & -20.63 \\ \hline
\end{tabular}
}
\caption{$C_{\Sigma}=-0.10$, $R_a=-1.0$.}
\label{GAvgJointUC01R1}
\end{small}
\end{table}
\footnotetext{The results are presented for the best $p$, found empirically.}

The Myopic-greedy approach attained utilities statistically significantly 
greater than those obtained by the heuristic case when $C_\Sigma=-0.1$ (see Table~\ref{GAvgJointUC01R1})\footnote{Statistical significance has been established with t-test.}.
When $C_\Sigma=-0.1$ and $P_u=0.4$, Myopic-greedy even attained utilities not significantly
different (with significance level 98\%) than Ideal.
%Ideal always attained higher utilities than Myopic-greedy, although when 
%$C_\Sigma=-0.1$ and $P_u=0.4$ both values were not significantly different with significance level 98\%.

\begin{table}[hbtp]
\begin{small}
\centerline{
\begin{tabular}{|l|l|l|l|l|l|}
\hline
&\multicolumn{5}{c|}{Average Joint Utility} \\ \cline{2-6}\multicolumn{1}{|c|}{$P_u$}
& \multicolumn{1}{|c|}{No-Comm.}&\multicolumn{1}{c|}{Ideal $C_\Sigma=0$} &\multicolumn{2}{c|}{Comm. SubGoals -- Best p} &\multicolumn{1}{c|}{Myopic-greedy} \\ \hline \hline
$0.2$ &-104.925 &{\bf -62.872} & -65.906\ \ \ \ \ \ \ \ \ \  &\ \ \ \ \ \ 0.3 &$$ -63.84\\ \hline
$0.4$ &-51.4522 &{\bf -37.33}  & -39.558\ \ \ \ \ \ \ \ \ \    &\ \ \ \ \ \ 0.2 &$$ -37.774 \\ \hline
$0.6$ &-33.4955  &{\bf -26.444}& -27.996\ \ \ \ \ \ \ \ \ \ &\ \ \ \ \ \ 0.2 & $$ -27.156 \\ \hline
$0.8$ &-24.3202 &{\bf -20.584} & -21.05\ \ \ \ \ \ \ \ \ \ &\ \ \ \ \ \ 0.1 &$$ -21.3  \\ \hline
\end{tabular}
}
\caption{$C_{\Sigma}=-1.0$ in SubGoals and Myopic-greedy, $R_a=-1.0$.}
\label{GAvgJointUC1R1}
\end{small}
\end{table}

\begin{table}[hbtp]
\begin{small}
\centerline{
\begin{tabular}{|l|l|l|l|l|l|}
\hline
&\multicolumn{5}{c|}{Average Joint Utility} \\ \cline{2-6}\multicolumn{1}{|c|}{$P_u$}
& \multicolumn{1}{|c|}{No-Comm.}&\multicolumn{1}{c|}{Ideal $C_\Sigma=0$} &\multicolumn{2}{c|}{Comm. SubGoals -- Best p} &\multicolumn{1}{c|}{Myopic-greedy}\\ \hline \hline
$0.2$ &-104.925 &{\bf -62.872} & -69.286\ \ \ \ \ \ \ \ \ \  &\ \ \ \ \ \ 0.1 &$$-68.948\\ \hline
$0.4$ &-51.4522 &{\bf -37.33}  & -40.516\ \ \ \ \ \ \ \ \ \    &\ \ \ \ \ \ 0.1 &$$ -40.594\\ \hline
$0.6$ &-33.4955  &{\bf -26.444}& -28.192\ \ \ \ \ \ \ \ \ \ &\ \ \ \ \ \ 0.1 & $$ -28.908\\ \hline
$0.8$ &-24.3202 &{\bf -20.584} & -21.118\ \ \ \ \ \ \ \ \ \ &\ \ \ \ \ \ 0.1 &$$ -22.166\\ \hline
\end{tabular}
}
\caption{$C_{\Sigma}=-10.0$ in SubGoals and Myopic-greedy, $R_a=-1.0$.}
\label{GAvgJointUC10R1}
\end{small}
\end{table}

When $C_\Sigma=-1$ (see Table~\ref{GAvgJointUC1R1}) the utilities attained for the 
Myopic-greedy approach when $P_u<0.8$ are significantly greater than the results obtained 
in the heuristic case. When $P_u=0.8$, the heuristic case was found
to be better than Myopic-greedy for the best choice of $p$ (Myopic-greedy obtained -21.3, and the 
SubGoals with $p=0.1$ attained -21.05 (variance=2.18)). 
The utilities attained by the Myopic-greedy agents, when $C_\Sigma=-10$ 
(see Table~\ref{GAvgJointUC10R1}) and 
$P_u$ in $\{0.2,0.4\}$, were not significantly different from the SubGoals case for 
the best $p$ with significance levels 61\% and 82\%, respectively. However, the 
heuristic case yielded smaller costs for the other values of $P_u =\{0.6, 0.8\}$.
One important point to notice is that these results consider the best $p$
found for the heuristic after trying a set of discrete values for p (see the x-axis
in Figure ~\ref{AvgUPr08}). In general trying and tuning  a heuristic parameter can be
very time consuming and the best choice may not be known ahead of time to the designer.
On the other hand, the Myopic-greedy approach does not require any tuning of any
parameter. In all the settings tested, Myopic-greedy always attain utilities higher 
than those attained in the SubGoals case with the worst $p$. 

\begin{table}[hbtp]
\begin{small}
\centerline{
\begin{tabular}{|l|l|l|l|l|}
\hline
&\multicolumn{4}{c|}{Average Communication Acts Performed} \\ \cline{2-5}\multicolumn{1}{|c|}{$P_u$}
& \multicolumn{1}{|c|}{No-Comm.}&\multicolumn{1}{c|}{Ideal $C_\Sigma=0$} &\multicolumn{1}{c|}{SubGoals} &\multicolumn{1}{c|}{Myopic-greedy} \\ \hline \hline
$0.2$ &0 &31.436 &5.4 &21.096 \\ \hline 
$0.4$ &0 &18.665 &1   &11.962 \\ \hline
$0.6$ &0 &13.426 &1   & 8.323 \\  \hline 
$0.8$ &0 &10.292 &1   & 4.579 \\ \hline
\end{tabular}
}
\caption{$C_{\Sigma}=-0.10$, $R_a=-1.0$.}
\label{GAvgCommC01R1}
\end{small}
\end{table}

\begin{table}[hbtp]
\begin{small}
\centerline{
\begin{tabular}{|l|l|l|l|l|}
\hline
&\multicolumn{4}{c|}{Average Communication Acts Performed} \\ \cline{2-5}\multicolumn{1}{|c|}{$P_u$}
& \multicolumn{1}{|c|}{No-Comm.}&\multicolumn{1}{c|}{Ideal $C_\Sigma=0$} &\multicolumn{1}{c|}{Comm. SubGoals} &\multicolumn{1}{c|}{Myopic-greedy}\\ \hline \hline
$0.2$ &0 &31.436 &1.194   &6.717\\ \hline
$0.4$ &0 &18.665 &1   &3.904    \\ \hline
$0.6$ &0 &13.426 &1   &2.036    \\  \hline
$0.8$ &0 &10.292 &0   &1.296    \\ \hline
\end{tabular}
}
\caption{$C_{\Sigma}=-1.0$ in Myopic-greedy and SubGoals, $R_a=-1.0$.}
\label{GAvgCommC1R1}
\end{small}
\end{table}

\begin{table}[hbtp]
\begin{small}
\centerline{
\begin{tabular}{|l|l|l|l|l|}
\hline
&\multicolumn{4}{c|}{Average Communication Acts Performed} \\ \cline{2-5}\multicolumn{1}{|c|}{$P_u$}
& \multicolumn{1}{|c|}{No-Comm.}&\multicolumn{1}{c|}{Ideal $C_\Sigma=0$} &\multicolumn{1}{c|}{Comm. SubGoals} &\multicolumn{1}{c|}{Myopic-greedy}\\ \hline \hline
$0.2$ &0 &31.436 &0   &0.416 \\ \hline
$0.4$ &0 &18.665 &0   &0.417 \\ \hline
$0.6$ &0 &13.426 &0   &0.338 \\  \hline
$0.8$ &0 &10.292 &0   &0.329 \\ \hline
\end{tabular}
}
\caption{$C_{\Sigma}=-10.0$ in Myopic-greedy and SubGoals, $R_a=-1.0$.}
\label{GAvgCommC10R1}
\end{small}
\end{table}
Tables~\ref{GAvgCommC01R1}, \ref{GAvgCommC1R1} and \ref{GAvgCommC10R1} present the average 
number of communication acts performed in each one of these cases.

\section{Discussion}
\label{discussion}

\begin{figure}[t]
%\centerline{\psfig{figure=ComplexityGraphMDPepsLevel2.prn,height=10cm}}
\centerline{\includegraphics[height=3.4in]{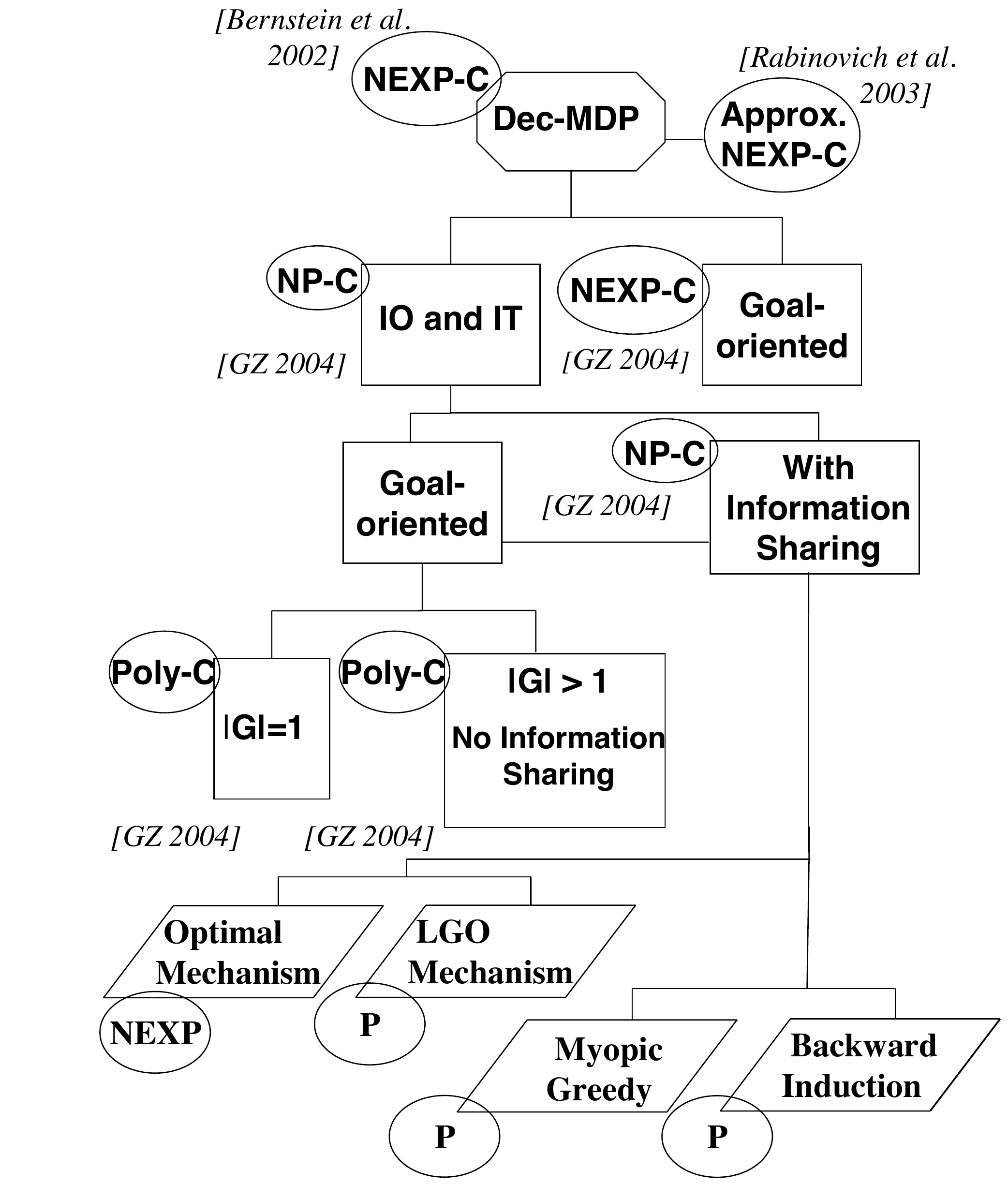}}
\caption{The complexity of solving Dec-MDPs.}
\label{resultsfig}
\end{figure}

Solving {\em optimally} decentralized control problems is known to be very hard. 
Figure~\ref{resultsfig} summarizes the complexity results (the rectangles
stand for optimal solutions while the parallelograms stand for solutions proposed in the framework
of communication-based decomposition mechanisms).
This taxonomy helps us understand the characteristics of different classes of
decentralized control problems and their effect on the {\em complexity} of these problems. 
The Coverage-set~\cite{Goldman04d}, Opt1Goal and OptNGoals~\cite{Goldman04c} were the first algorithms to solve optimally some non-trivial classes of Dec-MDPs. 

This paper presents communication-based decomposition mechanisms as a way 
to approximate the optimal joint solution of decentralized control problems.
This approach is based on two key ideas: (1) separating the questions of {\em when} to communicate 
and {\em what} to do between communications, and (2) exploiting the full observability of the global 
states after each communication to generate {\em individual} behaviors that the 
agents will follow between communications.
Communication between the decision makers serves as a synchronization point where local information is 
exchanged in order to assign an individual behavior to each controller. 
This addresses effectively applications in which constant communication is not desirable or not feasible. 
Many practical reasons could prevent agents from communication constantly.  Communication actions may
incur some costs that reflect the complexity of transmitting the information, the utilization of limited 
bandwidth that may be shared with other applications, or the risk of revealing information to competitive parties 
operating in the same environment.  Our communication-based decomposition mechanism divides the global problem into individual behaviors combined with communication acts to overcome the 
lack of global information.

We formalized communication-based decomposition mechanisms 
as a decentralized semi-Markov process with communication (Dec-SMDP-Com). 
We proved that solving optimally such problems with temporally abstracted actions is equivalent
to solving optimally a multi-agent MDP (MMDP).%, known to be P-complete. 
We adapted the multi-step backup policy iteration algorithm to the 
decentralized case that can solve the Dec-SMDP-Com problem optimally. 
This algorithm produces the optimal communication-based decomposition mechanism. 
% This is the first algorithm to tackle a general Dec-MDP-Com problem.
To provide a tractable algorithm, we can restrict the set of individual behaviors that are allowed. 
We proposed the $\mbox{LGO-MSBPI}$ polynomial-time algorithm that computes the assignment of a 
pair of local goals and period of time $k$ with the highest value to each possible global state. 
Adding local goals to the model seems more natural and intuitive than computing local behaviors based 
on general options. It is easier to state when a local behavior is completed when some local goal is reached, rather than stating sequences of local actions that eventually should achieve some desired
global behavior. Furthermore, an unrestricted set of options is larger and therefore it is computationally cheaper
to compute the decomposition mechanism when local goals are assumed. This intuition is confirmed by our experiments.  But the general question remain open, namely to determine when it is beneficial to compute local behaviors out of general options rather than assuming local goals.

The paper also presents a simpler approximation method. It assumes that a certain 
mechanism is given, i.e., human knowledge is incorporated into the model to provide 
agents with individual policies of actions (not including communication acts).
%each agent is given a local policy for any global state that is revealed through communication. 
%We showed two approaches to computing a policy of communication for a given decomposition mechanism. 
A greedy-approach is presented that computes the best time to communicate 
assuming there is only one opportunity for exchanging information.  The paper concludes with an empirical assessment of these approaches.

In summary, this paper contributes a communication-based
decomposition mechanism that can be applied to many of the hard
decentralized control problems shown in Figure~\ref{resultsfig}. 
This approach enables us to compute tractable individual 
behaviors for each agent together with the most
beneficial time to communicate and change these local behaviors.
The analytical results in the paper support the validity of the approach with respect to 
Dec-MDPs with independent transitions and observations. 
However, it is straightforward to apply the approach to general Dec-POMDPs or Dec-MDPs with dependent transitions or observations, and we believe it offers a viable approximation technique for these problems as well.
Our approach is scalable with respect to the number of agents since all the complexity results presented will increase linearly as more agents are added to the system.  Exchange of information is assumed to be via a broadcast to all the other agents. An interesting future extension is to study how agents can efficiently choose partners for communication to avoid global broadcasting.

\acks{This work was supported in part by the National Science Foundation
under grants IIS-0219606, by the Air Force Office of Scientific Research
under grants F49620-03-1-0090 and FA9550-08-1-0181, and
by NASA under grant NCC 2-1311.
Any opinions, findings, and conclusions or recommendations
expressed in this material are those of the authors and do not reflect
the views of the NSF, AFOSR or NASA.
}

%\bibliographystyle{plain}
%\bibliography{/nfs/coord/u0/clag/BIB/my}
\bibliography{my}
\bibliographystyle{theapa}

\end{document}